\pdfoutput=1

\documentclass[11pt]{article}

\usepackage[preprint]{acl}

\usepackage{fancyhdr}

\fancypagestyle{aclfooter}{%
  \fancyhf{}
  
  \fancyfoot[c]{
  \makebox[0pt][c]{
  \begin{minipage}{17cm}
    \centering
    \thepage\\
    \small
    \textit{Proceedings of the 62nd Annual Meeting of the Association for Computational Linguistics (Volume 1: Long Papers)}, pages 15135–15156\\
    August 11-16, 2024 ©2024 Association for Computational Linguistics
  \end{minipage}
  }}
}
\usepackage{times}
\usepackage{latexsym}

\usepackage[T1]{fontenc}

\usepackage[utf8]{inputenc}

\usepackage{microtype}

\usepackage{inconsolata}

\usepackage{amsmath}
\usepackage{wasysym}

\usepackage{tikz}
\usetikzlibrary{tikzmark,calc}
\usepackage{pifont}
\usepackage{changepage}
\usepackage{dsfont}
\usepackage{amssymb}
\usepackage{amsthm,bm}
\usepackage{thm-restate}
\usepackage{thmtools}
\usepackage{seqsplit}

\DeclareMathAlphabet{\mathpzc}{OT1}{pzc}{m}{it}
\usepackage{diagbox}

\makeatletter
\patchcmd{\math@cr@@@align}{\cr}{\global\let\df@label\@empty\cr}{}{}
\makeatother

\usepackage{environ}
\NewEnviron{scaledalign}{
\vspace{-13pt}
\begin{adjustwidth}{-.05\linewidth}{-.05\linewidth}
    \begin{gather*}
    \scalebox{.9}{$
      \begin{align}
        \BODY
      \end{align}
    $}
    \end{gather*}
\end{adjustwidth}
}
\NewEnviron{scaledalignnoadjust}{
    \begin{gather*}
    \scalebox{.9}{$
      \begin{align}
        \BODY
      \end{align}
    $}
    \end{gather*}
}

\makeatletter
\newcommand{\qualification}[1]{\ifthmt@thisistheone#1\fi}
\makeatother

\title{Tree-Averaging Algorithms for\\Ensemble-Based Unsupervised Discontinuous Constituency Parsing}

\author{Behzad Shayegh \quad\quad Yuqiao Wen \quad\quad Lili Mou$^{*}$ \\ Dept.~Computing Science \& Alberta Machine Intelligence Institute, University of Alberta\\
$^*$Canada CIFAR AI Chair\\
\texttt{\href{mailto:the.shayegh@gmail.com}{the.shayegh@gmail.com}}\quad \texttt{\href{mailto:yq.when@gmail.com}{yq.when@gmail.com}}\quad \texttt{\href{mailto:doublepower.mou@gmail.com}{doublepower.mou@gmail.com}}
}

\begin{document}

\setlength{\abovedisplayskip}{7pt}
\setlength{\belowdisplayskip}{7pt}
\setlength{\abovedisplayshortskip}{7pt}
\setlength{\belowdisplayshortskip}{7pt}

	\setlength{\textfloatsep}{10pt plus 1.0pt minus 2.0pt}
\setlength{\floatsep}{10pt plus 1.0pt minus 2.0pt}
\setlength{\intextsep}{10pt plus 1.0pt minus 2.0pt}

\maketitle
\begin{abstract}
We address unsupervised discontinuous constituency parsing, where we observe a high variance in the performance of the only previous model in the literature. We propose to build an ensemble of different runs of the existing discontinuous parser by averaging the predicted trees, to stabilize and boost performance. 
To begin with, we provide comprehensive computational complexity analysis (in terms of \textbf{P} and \textbf{NP}-complete) for tree averaging under different setups of binarity and continuity. We then develop an efficient exact algorithm to tackle the task, which runs in a reasonable time for all samples in our experiments. Results on three datasets show our method outperforms all baselines in all metrics; we also provide in-depth analyses of our approach.\footnote{Code available at \href{https://github.com/MANGA-UOFA/TAA4EUDCP}{https://github.com/MANGA-UOFA/ TAA4EUDCP}}
\end{abstract}

\section{Introduction}\label{sec:intro}

Unsupervised parsing has been attracting the interest of researchers over decades~\citep{klein-manning-2002-generative,klein2005unsupervised,snyder-etal-2009-unsupervised,shen2018ordered,cao-etal-2020-unsupervised}. Compared with supervised methods, unsupervised parsing has its own importance: (1) It reduces the reliance on linguistically annotated data, and is beneficial for low-resource languages and domains~\citep{kann-etal-2019-neural}. (2) Discovering language structures in an unsupervised way helps to verify theories in linguistics~\citep{10.1162/089120101750300490} and cognitive science~\citep{Exemplar2Grammar}. (3) Unsupervised parsing methods are applicable to other types of streaming data, such as motion-sensor signals~\citep{6137337}.

A constituency tree, a hierarchy of words and phrases shown in Figure~\ref{fig:disco_example}, is an important parse structure~\citep{constituencyTree}. Researchers have proposed various approaches to address unsupervised constituency parsing, such as latent-variable methods~\citep{clark-2001-unsupervised,petrov-klein-2007-improved,kim-etal-2019-compound} and rule-based systems~\citep{cao-etal-2020-unsupervised,li-lu-2023-contextual}. In our previous work~\citep{shayegh2023ensemble}, we reveal that different unsupervised parsers have low correlations with each other, and further propose an ensemble approach based on dynamic programming to boost performance. 

\begin{figure}[t]
\resizebox{\linewidth}{!}{
\renewcommand{\tabcolsep}{2pt}
    \centering\sffamily
    \begin{tabular}{c@{\hspace{.1cm}}c@{}c@{}c@{}c@{}c@{}c@{}c r@{\hspace{1.9cm}}|l@{\hspace{1.8cm}} c@{}c@{}c@{}c@{}c@{}c@{}c@{}c@{\hspace{.4cm}}}
    &
    {\Large (a)}&&&\tikzmarknode{t2w123}{\ding{71}}&&&
    &&&
    &&&\tikzmarknode{t1w123}{{\ding{71}}}&&&
    {\Large (b)}
    &\\&
    &\tikzmarknode{t2w12}{\hphantom{\ding{71}}}&{\hphantom{\ding{71}}}&{\hphantom{\ding{71}}}&
    \multicolumn{4}{l|}{\tikzmarknode{t2w23}{\ding{71}} $\rightarrow$ non-binary}
    &
    \multicolumn{3}{r}{fan-out 2 $\leftarrow$ \tikzmarknode{t1w12}{\ding{71}}}&{\hphantom{\ding{71}}}&{\hphantom{\ding{71}}}&\tikzmarknode{t1w23}{{\ding{71}}}&\tikzmarknode{t1w34}{\hphantom{\ding{71}}}&
    &\\&
    \tikzmarknode{t2w1}{\hphantom{\ding{71}}}&&\tikzmarknode{t2w2}{\hphantom{\ding{71}}}&&\tikzmarknode{t2w3}{\hphantom{\ding{71}}}&\hphantom{\ding{71}}&\tikzmarknode{t2w4}{\hphantom{\ding{71}}}
    &&&
    \tikzmarknode{t1w1}{\hphantom{\ding{71}}}&\hphantom{\ding{71}}&\tikzmarknode{t1w2}{\hphantom{\ding{71}}}&&\tikzmarknode{t1w3}{\hphantom{\ding{71}}}&&\tikzmarknode{t1w4}{\hphantom{\ding{71}}}
    &\\&
    \tikzmarknode{ww1}{Buy}&&\tikzmarknode{ww2}{the}&&\tikzmarknode{ww3}{pretty}&&\tikzmarknode{ww4}{book}
    &&&
    \tikzmarknode{w1}{Wake}&&\tikzmarknode{w2}{your}&&\tikzmarknode{w3}{friend}&&\tikzmarknode{w4}{up}
    &
\end{tabular}
    \begin{tikzpicture}[remember picture, overlay]
        \draw[-, black, thick] (t1w123) -- (t1w12);
        \draw[-, black, thick] (t1w123) -- (t1w23);
        \draw[-, black, thick] (t1w23) -- (t1w3);
        \draw[-, black, thick] (t1w23) -- (t1w2);
        \draw[-, black, thick] (t1w12) -- (t1w1);
        \draw[-, black, thick] (t1w12) -- (t1w4);
        
        \draw[-, black, thick] (t2w123) -- (t2w23);
        \draw[-, black, thick] (t2w123) -- (t2w1);
        \draw[-, black, thick] (t2w23) -- (t2w4);
        \draw[-, black, thick] (t2w23) -- (t2w3);
        \draw[-, black, thick] (t2w23) -- (t2w2);
    \end{tikzpicture}
}
\resizebox{\linewidth}{!}{\sffamily
\begin{tabular}{c c c c c c c}
    \hline\\
    {\Large (c)}&&&&\multicolumn{2}{c}{\tikzmarknode{t3w1234567}{{\ding{71}}} $\rightarrow$ non-binary fan-out 1}
    \\
    \multicolumn{4}{r}{binary fan-out 3 $\leftarrow$ \tikzmarknode{t3w13467}{{\ding{71}}}}
    \\
    \\
    \multicolumn{3}{r}{fan-out 3 $\leftarrow$ \tikzmarknode{t3w1346}{{\ding{71}}}}
    \\
    \\
    &&&\multicolumn{2}{l}{\tikzmarknode{t3w34}{{\ding{71}}} $\rightarrow$ fan-out 1}
    \\
    \tikzmarknode{t3w1}&\tikzmarknode{t3w2}{\hphantom{\ding{71}}}&\tikzmarknode{t3w3}{\hphantom{\ding{71}}}&\tikzmarknode{t3w4}{\hphantom{\ding{71}}}&\tikzmarknode{t3w5}{\hphantom{\ding{71}}}&\tikzmarknode{t3w6}{\hphantom{\ding{71}}}&\tikzmarknode{t3w7}{\hphantom{\ding{71}}}
    \\
    \tikzmarknode{www1}{Damit}&\tikzmarknode{www2}{sollen}&\tikzmarknode{www3}{den}&\tikzmarknode{www4}{Kassen}&\tikzmarknode{www5}{Beitragserhöhungen}&\tikzmarknode{www6}{erschwert}&\tikzmarknode{www7}{werden}
\end{tabular}
    \begin{tikzpicture}[remember picture, overlay]
        \draw[-, black, thick] (t3w1234567) -- (t3w13467);
        \draw[-, black, thick] (t3w1234567) -- (t3w2);
        \draw[-, black, thick] (t3w1234567) -- (t3w5);
        \draw[-, black, thick] (t3w13467) -- (t3w7);
        \draw[-, black, thick] (t3w13467) -- (t3w1346);
        \draw[-, black, thick] (t3w1346) -- (t3w6);
        \draw[-, black, thick] (t3w1346) -- (t3w1);
        \draw[-, black, thick] (t3w1346) -- (t3w34);
        \draw[-, black, thick] (t3w34) -- (t3w3);
        \draw[-, black, thick] (t3w34) -- (t3w4);
    \end{tikzpicture}
}
    \caption{(a) A continuous parse structure in English. (b) An arguably discontinuous parse structure in English. (c) A discontinuous parse structure in German. Interesting structures (binarity and fan-out) are illustrated.}
    \label{fig:disco_example}
\end{figure}
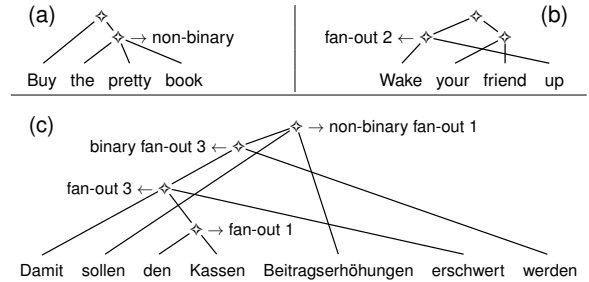

One limitation of most existing studies is that they only address continuous constituency parsing, that is, a constituent can only be a continuous span of words (Figure~\ref{fig:disco_example}a). However, a constituent may be discontinuous~\citep{Parentheticals1982}.\footnote{In the literature, a continuous parse structure is also known to be \textit{projective}, and a discontinuous parse structure is \textit{non-projective}~\citep{versley-2014-experiments}.} In Figure~\ref{fig:disco_example}b, for example, it is linguistically arguable that \textit{wake} and \textit{up} form a constituent in the sentence ``Wake your friend up.'' Such discontinuous constituents are more common in certain languages (such as German shown in Figure~\ref{fig:disco_example}c) than in English~\citep{skut-etal-1997-annotation}, but they have been less tackled in the literature. Very recently, \citet{yang-etal-2023-unsupervised} proposed an unsupervised parsing method based on a mildly context-sensitive grammar~\citep{joshi_1985} that allows discontinuous parse structures, known as discontinuous trees~\citep{10.5555/532941}. As an early attempt, their performance appears to be low, and we find that their approach is noisy and exhibits low correlation in different runs with random seeds. 

In this work, we propose an ensemble method to discontinuous constituency parsing by tree averaging, following our previous work~\citep{shayegh2023ensemble}, in which we address the ensemble of continuous constituency parsing by CYK-like dynamic programming~\citep{YOUNGER1967189}. However, the previous approach is not directly applicable to this paper because of the discontinuous structures in our setup. Consequently, the seemingly small change in the setting (continuous vs.~discontinuous) leads to a completely different landscape of problems, requiring advanced algorithmic analysis and development. 

Specifically, we first analyze the computational complexity of tree averaging under different setups, such as binarity and continuity. We show certain problems are in the \textbf{P} category, whereas general tree averaging is \textbf{NP}-complete. Then, we develop an algorithm utilizing the meet-in-the-middle technique~\citep{horowitz1974computing} with effective pruning strategies, making our search practical despite its exponential time complexity.

Our experiments on German and Dutch show that our approach largely outperforms previous work in terms of continuous, discontinuous, and overall $F_1$ scores. To the best of our knowledge, our algorithm is the first that can handle non-binary constituents in the setting of unsupervised constituency parsing. 
 
To sum up, our main contributions conclude: (1)~proposing an ensemble approach to unsupervised discontinuous constituency parsing, (2)~theoretically analyzing the computational complexity of various tree-averaging settings, and (3)~conducting experiments on benchmark datasets to verify the effectiveness of our approach. 
\section{Approach}

\subsection{Unsupervised Discontinuous Constituency Parsing}
\label{sec:approach_UDCP}

In linguistics, a \textit{constituent} is one or more words that act as a semantic unit in a hierarchical tree structure~\citep{constituencyTree}. Discontinuous constituents are intriguing, where a constituent is split into two or more components by other words~\citep[Figure~\ref{fig:disco_example};][]{Parentheticals1982, 10.5555/532941}.

Unsupervised discontinuous constituency parsing aims to induce---without using linguistically annotated data for training---a parse structure that may contain discontinuous constituents. As mentioned in \S\ref{sec:intro}, unsupervised parsing is an important research topic, as it potentially helps low-resource languages, the development of linguistic and cognitive theory, as well as the processing of non-textual streaming data.

\citet{yang-etal-2023-unsupervised} propose the only known unsupervised discontinuous parsing approach, based on linear context-free rewriting systems~\citep[LCFRS;][]{vijay-shanker-etal-1987-characterizing}, a type of mildly context-sensitive grammars~\citep{joshi_1985}, which can model certain discontinuous constituents. 
They focus on binary LCFRS and limit the \textit{fan-out} of their grammar to be at most $2$, known as \mbox{LCFRS-2} \citep{stanojevic-steedman-2020-span}, meaning that each constituent can contain up to two nonadjacent components\footnote{A \textit{component} refers to a span of one or more consecutive words in a sentence.} (illustrated in Figure~\ref{fig:disco_example}). \citet{maier-etal-2012-plcfrs} observe such structures cover most cases in common treebanks. Essentially, the \mbox{LCFRS-2} grammar used in \citet{yang-etal-2023-unsupervised} is a 6-tuple $\mathcal{G} = (S, \mathcal{N}_1, \mathcal{N}_2, \mathcal{P}, \Sigma, \mathcal{R})$, where $S$ is the start symbol, $\mathcal{N}_1$ a finite set of non-terminal symbols with fan-out being $1$, $\mathcal{N}_2$ fan-out being $2$, $\mathcal{P}$ preterminals, and $\Sigma$ terminals. $\mathcal{R}$ is a finite set of rules, following one of the forms:
\begin{align}
\resizebox{.9\linewidth}{!}{$
\begin{aligned}
    \notag &R_1: S(x) \rightarrow A(x) &A\in\mathcal{N}_1\\
    \notag &R_2: A(xy) \rightarrow U(x)U'(y) &A\in\mathcal{N}_1\\
    \notag &R_3: A(xyz) \rightarrow U(y)B(x, z) &A\in\mathcal{N}_1\\
    \notag &R_4: A(x, y) \rightarrow U(x)U'(y) &A\in\mathcal{N}_2\\
    \notag &R_5: A(xy, z) \rightarrow U(x)B(y, z) &A\in\mathcal{N}_2\\
    \notag &R_6: A(xy, z) \rightarrow U(y)B(x, z) &A\in\mathcal{N}_2\\
    \notag &R_7: A(x, yz) \rightarrow U(y)B(x, z) &A\in\mathcal{N}_2\\
    \notag &R_8: A(x, yz) \rightarrow U(z)B(x, y) &A\in\mathcal{N}_2\\
    \notag &R_9: T(w) \rightarrow w &T\in\mathcal{P}, w\in\Sigma
\end{aligned}
$}
\end{align}
for $B\in\mathcal{N}_2$, and $U,U'\in\mathcal{N}_1\cup\mathcal{P}$. Note that $x$ in a rewriting rule $A(x)$ is not an input string. Instead, it is a placeholder suggesting that $A(x)$ has a fan-out $1$. Likewise, $A(x,y)$ has a fan-out $2$, with $x$ and $y$ being the two nonadjacent components. 

Such a grammar may handle certain types of discontinuous constituents. Take $R_3$ as an example. The left-hand side $A(xyz)$ implies the placeholder string $xyz$ is a constituent that can be split into three adjacent components: $x$, $y$, and $z$ in order. The notation $B(x,z)$ in the rule asserts that $x$ and~$z$ form a constituent, although they are not adjacent in the original string~$xyz$.

We follow \citet{yang-etal-2023-unsupervised} and train a probabilistic LCFRS-2, parametrized by a tensor
decomposition-based neural network (TN-LCFRS). The  objective is to maximize the likelihood of sentence reconstruction by marginalizing the grammar rules. Compared with general context-sensitive grammar, LCFRS-2 balances the modeling capacity and polynomial-time inference efficiency; thus, it is also used in supervised discontinuous parsing~\citep{maier-2010-direct,DOPDiscoCFT}. We refer interested readers to \citet{yang-etal-2023-unsupervised} for the details of TN-LCFRS training and inference.  

\begin{figure}[t]
\begin{center}
\includegraphics[width=0.62\linewidth]{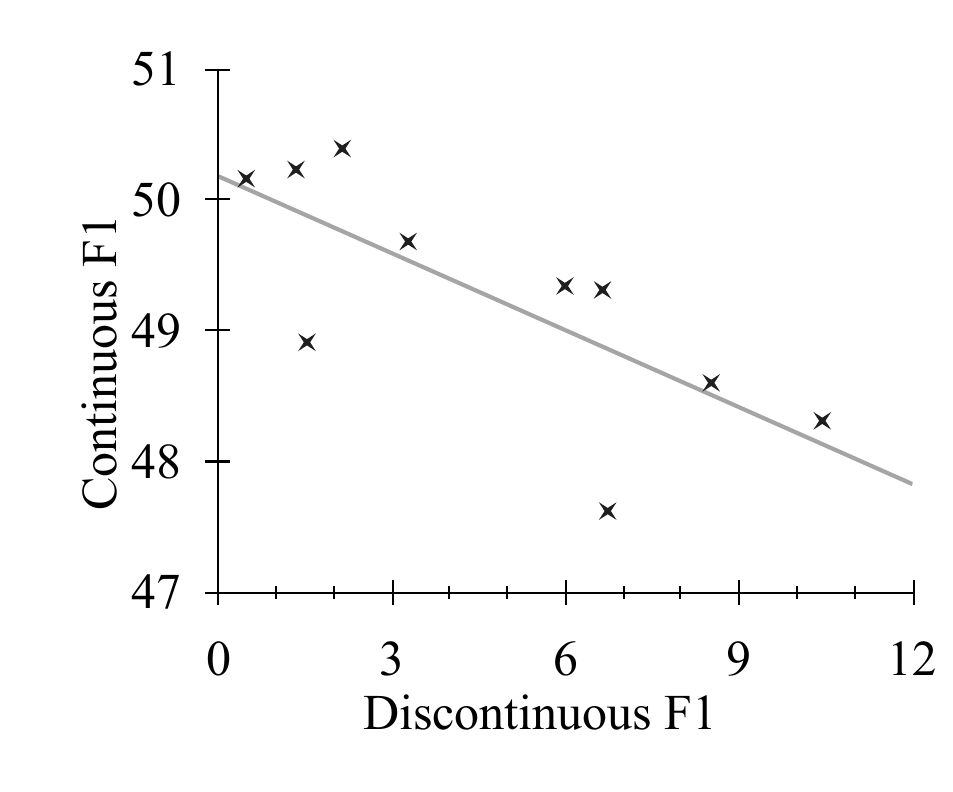}
\end{center}\vspace{-16pt}
\caption{$F_1$ scores on continuous and discontinuous constituents in the NEGRA test set~\citep{skut-etal-1997-annotation}, where each point is a random run of TN-LCFRS~\citep{yang-etal-2023-unsupervised}.}
\label{fig:cf1vsdf1}
\end{figure}

In our work, we observe that TN-LCFRS exhibits high variance in performance, especially for discontinuous constituents (Figure~\ref{fig:cf1vsdf1}). In addition, there is a negative correlation between performance on continuous and discontinuous constituents among different runs: some are better at continuous constituents, while others are better at discontinuous ones. 

In our previous work~\citep{shayegh2023ensemble}, we show that an ensemble model can utilize different expertise of existing continuous parsers and smooth out their noise. A natural question now is: \textit{Can we build an ensemble of \textbf{discontinuous} constituency parsers?}

In the rest of this section, we perform in-depth theoretical analysis of the problem, and show that with seemingly harmless alteration of setups, the problem may belong to either \textbf{P} or \textbf{NP}-complete complexity categories. 
We further develop a meet-in-the-middle search algorithm with efficient pruning to solve our problem efficiently in practice.

\subsection{Averaging over Constituency Trees}

We adopt the tree-averaging notion in~\citet{shayegh2023ensemble}, which suggests the ensemble output should be the tree with the highest average $F_1$-score against the ensemble components (referred to as \textit{individuals}): 
\begin{scaledalign}
T^* = \operatorname*{argmax}_{T \in \mathcal T} \sum_{k=1}^K F_1(T, T_k)
\label{eq:AvgTree}
\end{scaledalign}
where $\mathcal T$ is the search space given an input sentence, and $T_k$ is the parse tree predicted by the $k$th individual. The $F_1$ score is commonly used for evaluating constituency parsers, and is also our ensemble objective, given by
$
F_1(T_{\text{prd}}, T_{\text{ref}}) = \frac{2 |C(T_{\text{prd}}) \cap C(T_{\text{ref}})|}{|C(T_{\text{prd}})| + |C(T_{\text{ref}})|}
$,
where $T_{\text{prd}}$ and $T_{\text{ref}}$ denote the predicted and reference trees, respectively, while $C(T)$ represents the set of constituents in a tree $T$.

It is noted that existing unsupervised parsers can only produce binary trees in both continuous and discontinuous settings~\citep{shen2017neural,kim-etal-2019-compound,yang-etal-2023-unsupervised}. The binary property asserts that, given a length-$n$ sentence, $|C(T_k)|=2n-1$ for every individual~$k$, thus simplifying Eqn.~\eqref{eq:AvgTree} to
\begin{scaledalign}
\label{eq:SimpliedAvg}
T^* = \operatorname*{argmax}_{T \in \mathcal T}  \frac{\sum_{c \in C(T)} {\mathpzc h}(c)}{|C(T)|+2n-1}
\end{scaledalign}
Here, ${\mathpzc h}(c)$ counts the occurrences of a constituent~$c$ in the trees predicted by the individuals. We call ${\mathpzc h(c)}$ the \textit{hit count} of $c$. 

We point out that the output of our approach does not have to be a binary tree, and we will empirically analyze the output binarity in \S\ref{sec:results}.

We would like to examine the computational complexity for tree averaging, such as \textbf{P} and \textbf{NP} categories. To begin with, we consider the decision problem (i.e., whether there exists a tree satisfying some conditions) corresponding to the search problem (i.e., finding the best tree),  which is standard in complexity analysis~\citep{arora2009computational}.

\begin{restatable}[Averaging binary trees with bounded fan-out]{problem}{problemboundeddiscontinuous} \label{problem:boundeddiscontinuous}
Consider a number $z$ and constituency trees $T_1, \cdots,$ $T_K$ with the same leave nodes, where the trees are binary and have a fan-out of at most $F$. Is there a constituency tree $T$ such that $\sum_{k=1}^K F_1(T, T_k) \geq z$?
\end{restatable}

\begin{restatable}{theorem}{theoremboundeddiscontinuous} \label{thm:boundeddiscontinuous}
Problem~\ref{problem:boundeddiscontinuous} belongs to  \normalfont{\textbf{P}}.
\end{restatable}

\begin{proof}[Proof sketch]
In our previous work~\citep{shayegh2023ensemble}, we present a polynomial-time dynamic programming (DP) algorithm to average continuous binary trees, where the outputs are restricted to binary trees as well; here, continuous trees can be seen as having max fan-out~$1$. For non-binary outputs, the  DP table can be augmented by an additional axis whose size is bounded by the sentence length $n$. To handle discontinuity, we may augment the DP table with additional axes based on the maximum fan-out $F$.  
Overall, the time complexity of the DP algorithm is $\mathcal O(n^{4F+1})$; thus, Problem~\ref{problem:boundeddiscontinuous} is in \textbf{P}. See Appendix~\ref{Apndx:proof:boundeddiscontinuous} for the detailed proof.\footnote{\citet{corro-2023-dynamic} develops a CYK-like algorithm for span-based nested named-entity recognition that takes discontinuous and non-binary structures into account. However, their algorithm  limits its search space to constituents that contain at most one smaller multi-word constituent inside, making it not applicable to our scenario.}
\end{proof}

A follow-up question then is whether there is a polynomial-time algorithm when we relax the assumptions of input being binary and having bounded fan-out. Having non-binary inputs is intriguing since the above DP relies on Eqn.~\eqref{eq:SimpliedAvg}, which only holds for binary inputs.

\begin{restatable}[Averaging trees with bounded fan-out]{problem}{problemnonbinaryboundeddiscontinuous} \label{problem:nonbinaryboundeddiscontinuous}
Consider a number $z$ and constituency trees $T_1, \cdots,$ $T_K$ with the same leave nodes, where the fan-out is at most $F$ but the trees may be non-binary. Is there a constituency tree $T$ such that $\sum_{k=1}^K F_1(T, T_k) \geq z$?
\end{restatable}

\begin{restatable}{theorem}{theoremnonbinaryboundeddiscontinuous} \label{thm:nonbinaryboundeddiscontinuous}
Problem~\ref{problem:nonbinaryboundeddiscontinuous} belongs to  \normalfont{\textbf{P}}.
\end{restatable}

\begin{proof}[Proof sketch]
If we fix the number of nodes $\tau$ of the output tree, the above DP can be reused to solve this problem. We may enumerate all possible values of $\tau$, which must satisfy $n < \tau < 2n$. 
Thus, Problem~\ref{problem:nonbinaryboundeddiscontinuous} can be solved in $\mathcal{O}(n^{4F+2})$ time and belongs to~\textbf{P}. See Appendix~\ref{Apndx:proof:nonbinaryboundeddiscontinuous} for the proof.
\end{proof}

However, if the fan-out is unbounded, the difficulty is that the above DP table grows exponentially with respect to $F$. Since the problem is polynomial-time verifiable, it surely belongs to \textbf{NP}, but whether it belongs to \textbf{P}, \textbf{NP}-complete, or both remains unknown for binary inputs.

\begin{restatable}[Averaging binary trees]{openproblem}{openproblembinarydisco} \label{openproblem:binarydisco}
What is the exact complexity category (\textbf{P} or \textbf{NP}-complete) of averaging binary trees with unbounded fan-out?
\end{restatable}

For averaging general trees (non-binary inputs with unbounded fan-out), we can show that it belongs to \textbf{NP}-complete.

\begin{restatable}[Averaging trees]{problem}{problemnonbinarydisco} \label{problem:nonbinarydisco}
Consider a number~$z$ and constituency trees $T_1, \cdots,$ $T_K$ with the same leave nodes, where the fan-out is unbounded and the trees may be non-binary. Is there a constituency tree $T$ such that $\sum_{k=1}^K F_1(T, T_k) \geq z$?
\end{restatable}

\begin{restatable}{theorem}{theoremnonbinarydisco}\label{thm:nonbinarydisco}
Problem~\ref{problem:nonbinarydisco} belongs to \normalfont{\textbf{NP}\textit{-complete}}.
\end{restatable}

\begin{proof}[Proof sketch]
It is easy to show that, given a certificate, Problem~\ref{problem:nonbinarydisco} is polynomial-time verifiable. For the completeness, we reduce the max clique problem, a known \textbf{NP}-complete problem, to Problem~\ref{problem:nonbinarydisco}. See Appendix~\ref{Apndx:proof:nonbinarydisco} for the detailed proof.
\end{proof}

\begin{table}[t]
\centering
\resizebox{.9\linewidth}{!}{
\begin{tabular}{|l|c|c|}
\hline
  \diagbox{Binary\\[-3pt]individuals?}{Bounded\\[-3pt] fan-out?\\[-11pt]} & Bounded & Unbounded\\
  \hline
Binary & \textbf{P} & Unknown \\ 
  \hline
Non-binary & \textbf{P} & \textbf{NP}-complete \\
\hline
\end{tabular}

}\vspace{-.1cm}
\caption{Summary of the complexity categories of tree-averaging problems.}
\label{table:complexity_conclusion}
\end{table}

The above theoretical analysis provides a deep understanding of building tree ensembles, summarized in Table~\ref{table:complexity_conclusion}. In our experiments, the task falls into Problem~\ref{problem:boundeddiscontinuous} because the only existing unsupervised discontinuous parser~\citep{yang-etal-2023-unsupervised} produces binary trees with fan-out at most $2$. The problem belongs to \textbf{P}.

That being said, a practical approach should consider not only algorithmic complexity, but also the specific properties of the task at hand. A DP algorithm that one may develop for our task has a complexity of $\mathcal{O}(n^9)$, as discussed in the proof of Theorem~\ref{thm:boundeddiscontinuous}, which appears impractical despite its polynomial time complexity. 

In the rest of this section, we develop a more general search algorithm that works at the level of  Open Problem~\ref{openproblem:binarydisco} but also solves Problem~\ref{problem:boundeddiscontinuous} more efficiently in practice than high-order polynomial DP.

\subsection{Our Search Algorithm}\label{sec:our_search_algo}

Our work concerns building an ensemble of binary, bounded-fan-out, discontinuous trees generated by the unsupervised parser in \citet{yang-etal-2023-unsupervised}. 

We will develop a general search algorithm (regardless of fan-out) with strong pruning that only needs to consider a few candidate constituents, bringing down the $\mathcal O(2^{2^n})$ complexity of exhaustive search\footnote{For a length-$n$ sentence, there are $2^n-1$ possible constituents because an arbitrary non-empty set of words can be a constituent. To build a constituency tree, the exhaustive search needs to look into any combination of constituents.} to $\mathcal O(2^n)$. We will further halve the exponent using a meet-in-the-middle technique, resulting in an overall complexity of $\mathcal O(2^{\frac{n}{2}}n^2)$. As a result, our algorithm remains exact search while being efficient for all samples in the dataset; when a sentence is short, our algorithm is even much faster than the DP with $\mathcal{O}(n^9)$ complexity.

To solve the tree-averaging problem, we first convert it into an equivalent graph problem for clarity. Specifically, we construct an undirected graph $G$ with vertices being all possible constituents (although they can be largely pruned), weighted by corresponding hit counts, as defined after Eqn.~\eqref{eq:SimpliedAvg}. We put an edge between two vertices if and only if their corresponding constituents can coexist within a constituency tree, that is to say, they are either disjoint or inclusive of one another. Then, we formulate the below graph-search problem, which can solve the original tree-search problem for ensemble discontinuous parsing.

\begin{restatable}[Normalized Max Weighted Clique]{problem}{problemNMWC}\label{problem:NMWC}
    Consider a weighted undirected graph $G=\langle V, E\rangle$ and an objective function $f(Q; \alpha_1,\alpha_2) = \frac{\sum_{v\in Q} w(v)+\alpha_1}{|Q|+\alpha_2}$, where $\alpha_1,\alpha_2\in\mathbb R$, $Q\subseteq V$,  and $w(v)$ is the weight of $v$. What is the clique $Q$ that maximizes $f(Q; \alpha_1,\alpha_2)$?
\end{restatable}

\begin{restatable}{lemma}{lemmaNMWC}\label{lemma:NMWC}
    Given $G$ as described above, the clique $Q^*$ maximizing $f(Q; \alpha_1, \alpha_2)$ corresponds to a constituency tree, if $\alpha_1 \leq K\alpha_2$ where $K$ is the number of individuals.
\qualification{ (Proof in Appendix~\ref{Apndx:proof:lemmaNMWC}.)}
\end{restatable}

\begin{restatable}{theorem}{theoremNMWC}\label{thm:NMWC}
    The average constituency tree $T^*$ in Eqn.~(\ref{eq:SimpliedAvg}) is polynomial-time reducible from $Q^*$ in Problem~\ref{problem:NMWC} with $\alpha_1=0$ and $\alpha_2=2n-1$, if $G$ is built as above.
\end{restatable}

\begin{proof}
Suppose the constituents corresponding to a clique $Q$ in the graph $G$ form a constituency tree, denoted by $T_Q$. We have $\sum_{v \in Q} w(v) = \sum_{c\in T_Q} {\mathpzc h}(c)$ by the construction of $G$. Thus, the objective function of Problem~\ref{problem:NMWC} is
\begin{scaledalign}
f(Q; 0, 2n-1) &= \frac{\sum_{v\in Q} w(v) + 0}{|Q|+2n-1}\\
&= \frac{\sum_{c \in C(T_Q)} {\mathpzc h}(c)}{|C(T_Q)|+2n-1}
\end{scaledalign}
which is the same as the $\operatorname{argmax}$ objective in Eqn.~\eqref{eq:SimpliedAvg}.
It is easy to see that each constituency tree corresponds to a clique in $G$. Therefore, the corresponding constituency tree to $Q^*$, guaranteed by Lemma~\ref{lemma:NMWC}, maximizes Eqn.~\eqref{eq:SimpliedAvg}. In other words, we have $T^*=T_{Q^*}$.
\end{proof}

Problem~\ref{problem:NMWC} generalizes a standard max clique problem~\citep{arora2009computational}, which may be solved in $\mathcal{O}(2^{|V|})$ time complexity by exhaustive search. The meet-in-the-middle technique~\citep{horowitz1974computing} can be used to address the max clique problem, reducing the complexity from $\mathcal{O}(2^{|V|})$ to $\mathcal{O}(2^\frac{|V|}{2}|V|^2)$. In Appendix~\ref{apndx:meetinthemiddle}, we develop a variant of the meet-in-the-middle algorithm to solve Problem~\ref{problem:NMWC}.

\subsection{Candidate Constituents Pruning}\label{sec:pruning}
The efficiency of our algorithm for Problem~\ref{problem:NMWC} depends on the number of vertices in the graph. In our construction, the vertices in $G$ correspond to possible constituents, which we call \textit{candidates}. 

A na\"ive approach may consider all  $\mathcal O(2^n)$ possible constituents, which are non-empty combinations of words in a length-$n$ sentence. Thus, our meet-in-the-middle algorithm has an overall  complexity of $\mathcal O(2^{2^{n-1}+2n})$ for tree averaging. 

In this part, we theoretically derive lower and upper bounds for a candidate's hit count. If a constituent has a lower hit count than the lower bound, it is guaranteed not to appear in the average tree. On the other hand, a constituent must appear in the average tree, if it has a higher hit count than the upper bound. We may exclude both cases from our search process, and directly add the must-appear candidates to the solution in a \textit{post hoc} fashion.

Let $P$ be a set of constituents that are known to be in the average tree $T^*$. Here, $P$ may even be an empty set. We may derive a hit-count lower bound for other constituents (not in $P$), as stated in the following theorem.
\begin{restatable}[Lower bound]{theorem}{theoremlowerbound}\label{thm:lowerbound}
For every constituent $c\in C(T^*)\backslash P$, where $P\subseteq C(T^*)$, we have
\begin{scaledalign}
{\mathpzc h}(c) > \min_{|P| \leq j \leq 2n-2} \frac{\sum_{i=1}^{j-|P|} \lambda^{+}_i \;+\; \sum_{c' \in P} {\mathpzc h}(c')}{j+2n-1}
\end{scaledalign}
where $\lambda^{+}_i$ is the $i$th smallest positive hit count.
\qualification{ (Proof in Appendix~\ref{Apndx:proof:lowerbound}.)}
\end{restatable}

The theorem suggests that constituents having a hit count lower than the above threshold can be removed from the search. If we set $P=\emptyset$, we have $\mathpzc h(c)>0$, immediately pruning all zero-hit constituents. In other words, we may only consider the constituents appearing in at least one individual, which largely reduces the graph size from $\mathcal O(2^n)$ to $\mathcal O(nK)$ for $K$ individuals. In fact, the graph can be further pruned with the below theorem.

\begin{restatable}[Upper bound]{theorem}{theoremupperbound}\label{thm:upperbound}
   Let $c$ be a constituent with a hit count of $K$, where $K$ is the number of individuals. (a) The constituent $c$ is compatible---i.e., may occur in the same constituency tree---with every constituent in the average tree. (b) The constituent $c$ appears in the average tree.  
\qualification{ (Proof in Appendix~\ref{Apndx:proof:upperbound}.)}
\end{restatable}

Theorem~\ref{thm:upperbound}b suggests that the search process may exclude the constituents that occur in all individuals, denoted by $P=\{c: \mathpzc h(c)=K\}$. In this case, we may solve a reduced version of Problem~\ref{problem:NMWC} with the pruned graph and $\alpha_1 = \sum_{c' \in P} {\mathpzc h}(c')$, $\alpha_2 = |P|+2n-1$. The $\alpha_1$ and $\alpha_2$ in Theorem~\ref{thm:NMWC} can be modified accordingly.\footnote{The proof is parallel to that of the original version of Theorem~\ref{thm:NMWC}, and we leave it as an exercise for readers.} When we add $P$ back to the solution, the connectivity of $P$ is guaranteed by Theorem~\ref{thm:upperbound}a, forming the solution (guaranteed to be a clique) for the original Problem~\ref{problem:NMWC}.

To analyze the worst-case complexity, we notice that single words are constituents that must occur in all individuals. At least, we may set $P=\{c: c \text{ is a single word}\}$, and Theorem~\ref{thm:lowerbound} yields a lower bound of
\begin{scaledalign}
\label{eq:naive_lower_bound}
{\mathpzc h}(c) > \min_{n \leq j \leq 2n-2}\; \frac{0 + nK}{j+2n-1} > \frac{nK}{4n} = \frac{K}{4}
\end{scaledalign}
Then, we can find an upper bound for the number of candidates kept for the search process. Since the sum of all hit counts is always $(2n-1)K$, the number of candidates with a higher hit count than the threshold in Eqn.~\eqref{eq:naive_lower_bound} is bounded by
\begin{scaledalign}
|V| \leq \Big\lfloor \frac{(2n-1)K}{K/4} \Big\rfloor < 8n
\end{scaledalign}
This shows that our pruning mechanism, profoundly, reduces the graph size from $\mathcal{O}(2^n)$ to $\mathcal{O}(n)$, which in turn reduces the overall time complexity of our tree averaging from $\mathcal O(2^{2^{n-1}+2n})$ to $\mathcal{O}(2^\frac{n}{2}n^2)$ in the worst-case scenario.

In practice, the pruning is even more effective when the individuals either largely agree or largely disagree with each other. In the former case, we will have more all-hit constituents excluded from the search and consequently the lower bound increases, whereas in the latter case, many constituents will fall short of the lower bound. Empirically, our algorithm is efficient for all samples in our experiments.

\section{Experiments}

\subsection{Settings}
\label{sec:settings}

\textbf{Datasets.}
We evaluated our method on Dutch and German datasets, where discontinuous constituents are relatively common. In particular, we used the LASSY treebank~\citep{vanNoord2013} for Dutch. For German, we trained our individuals on the union of NEGRA~\citep{skut-etal-1997-annotation} and TIGER~\citep{TIGER_TB} treebanks, while testing and reporting the performance on their test sets respectively. Our settings strictly followed~\citet{yang-etal-2023-unsupervised} for fair comparison.

\textbf{Evaluation Metrics.} $F_1$ scores are commonly used for evaluating constituency parsing performance~\citep{klein2005unsupervised, shen2017neural, shen2018ordered, kim-etal-2019-compound, kim-etal-2019-unsupervised}. In our work of unsupervised discontinuous constituency parsing, we report corpus-level $F_1$ scores of all constituents, 
continuous constituents, and discontinuous constituents, denoted by $F_1^{\text{overall}}$,
$F_1^{\text{cont}}$, and
$F_1^{\text{disco}}$, respectively. We followed the previous work~\citep{yang-etal-2023-unsupervised}, which discards punctuation and excludes trivial constituents (the whole sentence and single words) when calculating the $F_1$ scores.

\textbf{Setups of Our Ensemble.} As mentioned in \S\ref{sec:approach_UDCP}, we used different runs of the tensor decomposition-based neural LCFRS~\cite[TN-LCFRS;][]{yang-etal-2023-unsupervised} as our individual models. Our experiment was based on the released code\footnote{https://github.com/sustcsonglin/TN-LCFRS} with default training hyperparameters. However, we find the TN-LCFRS is highly unstable, with a very high variance and lower overall performance than \citet{yang-etal-2023-unsupervised}. As a remedy, we trained the model $10$ times and selected the top five based on validation $F_1^\text{overall}$, and this yields performance close to \citet{yang-etal-2023-unsupervised}. In addition, we observe that random weight initializations lead to near-zero correlation of the predicted discontinuous constituents, but our ensemble method expects the individuals to at least agree with each other to some extent. Therefore, our different runs started with the same weight initialization\footnote{Our pilot study shows the proposed ensemble method is not sensitive to the initialization, as long as it is shared among different runs. This is also supported by the evidence that our approach performs well consistently on three datasets.} but randomly shuffled the order of training samples to achieve stochasticity.

Our ensemble method does not have hyperparameters. However, $F_1^{\text{overall}}$ and  $F_1^{\text{disco}}$ may not correlate well. Since our ensemble objective is solely based on $F_1^{\text{overall}}$, we may enhance $F_1^{\text{disco}}$ by weighting the individuals with validation $F_1^{\text{disco}}$ scores in hopes of achieving high performance in all aspects. Note that weighting individuals does not hurt our theoretical analysis and algorithm, because weighting is equivalent to duplicating individuals (since all the weights are rational numbers) and can be implemented by modifying hit counts without actual duplication.

\begin{table*}[t]
\centering
\resizebox{\linewidth}{!}{
\begin{tabular}{|r l | l l l | l l l | l l l|}
\hline
& & \multicolumn{3}{c|}{NEGRA} & \multicolumn{3}{c|}{TIGER} & \multicolumn{3}{c|}{LASSY}\\
\multicolumn{2}{|c|}{Method$^{(\#\text{ preterminal symbols})}$} & $F_1^{\text{overall}}$ & $F_1^{\text{cont}}$ & $F_1^{\text{disco}}$ & $F_1^{\text{overall}}$ & $F_1^{\text{cont}}$ & $F_1^{\text{disco}}$ & $F_1^{\text{overall}}$ & $F_1^{\text{cont}}$ & $F_1^{\text{disco}}$\\
\hline\hline
\multicolumn{2}{|l|}{Baselines (four runs each)$^\dag$}&&&&&&&&&\rule{0pt}{2.3ex}\\
1& Left branching & $\ \ 7.8$ & \ \ -- & $\ \ 0.0$ & $\ \ 7.9$ & \ \ -- & $\ \ 0.0$ & $\ \ 7.2$ & \ \ -- & $\ \ 0.0$\\
2& Right branching & $12.9$ & \ \ -- & $\ \ 0.0$ & $14.5$ & \ \ -- & $\ \ 0.0$ & $24.1$ & \ \ -- & $\ \ 0.0$\\
3& N-PCFG$^{(45)}$ & $40.8_{\pm0.5}$ & \ \ -- & $\ \ 0.0$ & $39.5_{\pm0.4}$ & \ \ -- & $\ \ 0.0$ & $40.1_{\pm3.9}$ & \ \ -- & $\ \ 0.0$\\
4& C-PCFG$^{(45)}$ & $39.1_{\pm1.9}$ & \ \ -- & $\ \ 0.0$ & $38.8_{\pm1.3}$ & \ \ -- & $\ \ 0.0$ & $37.9_{\pm3.4}$ & \ \ -- & $\ \ 0.0$\\
5& TN-PCFG$^{(4500)}$ & $45.4_{\pm0.5}$ & \ \ -- & $\ \ 0.0$ & $44.7_{\pm0.6}$ & \ \ -- & $\ \ 0.0$ & $44.3_{\pm6.4}$ & \ \ -- & $\ \ 0.0$\\
6& N-LCFRS$^{(45)}$ & $33.7_{\pm2.8}$ & \ \ -- & $\ \ 2.0_{\pm0.8}$ & $32.7_{\pm1.8}$ & \ \ -- & $\ \ 1.2_{\pm0.8}$ & $36.9_{\pm1.5}$ & \ \ -- & $\ \ 0.9_{\pm0.8}$\\
7& C-LCFRS$^{(45)}$ & $36.7_{\pm1.5}$ & \ \ -- & $\ \ 2.7_{\pm1.4}$ & $35.2_{\pm1.2}$ & \ \ -- & $\ \ 1.7_{\pm1.1}$ & $36.9_{\pm3.7}$ & \ \ -- & $\ \ 2.2_{\pm1.0}$\\
8& TN-LCFRS$^{(4500)}$ & $46.1_{\pm1.1}$ & \ \ -- & $\ \ 8.0_{\pm1.1}$ & $45.4_{\pm0.9}$ & \ \ -- & $\ \ 6.1_{\pm0.8}$ & $45.6_{\pm2.3}$ & \ \ -- & $\ \ 8.9_{\pm1.5}$\\
\hline
\multicolumn{2}{|l|}{Individuals: TN-LCFRS$^{(4500)}$}&&&&&&&&&\rule{0pt}{2.3ex}\\
9& Five runs & $46.4_{\pm0.5}$ & $49.8_{\pm1.3}$ & $\ \ 6.0_{\pm4.0}$ & $45.8_{\pm1.3}$ & $49.9_{\pm1.1}$ & $\ \ 4.0_{\pm3.2}$ & $46.7_{\pm2.0}$ & $50.9_{\pm1.7}$ & $\ \ 6.2_{\pm1.9}$\\
10& $F_1^{\text{overall}}$-best run & $\underline{46.9}$ & $50.2$ & $\ \ 1.3$ & $\underline{47.2}$ & $51.1$ & $\ \ 5.9$ & $\underline{48.2}$ & $52.4$ & $\ \ 5.8$\\
11& $F_1^{\text{cont}}$-best run & $46.7$ & $\underline{51.3}$ & $\ \ 7.3$ & $47.2$ & $\underline{51.1}$ & $\ \ 5.9$ & $48.2$ & $\underline{52.4}$ & $\ \ 5.8$\\
12& $F_1^{\text{disco}}$-best run & $46.0$ & $48.3$ & $\underline{10.4}$ & $45.4$ & $48.8$ & $\ \ \underline{6.6}$ & $48.0$ & $52.1$ & $\ \ \underline{8.6}$\\
\hline
13& Binary ensemble & ${47.6}^{*}$ & ${50.1}$  & $\ \ {9.9}$ & ${47.8}^{**}$ & ${51.5}^{}$  & $\ \ {6.5}$ & ${50.9}^{**}$ & ${54.6}^{**}$  & $\ \ {9.7}^{**}$\\
14& Non-binary ensemble & $\textbf{49.1}^{**}$ & $\textbf{51.5}$  & $\textbf{10.6}$ & $\textbf{48.7}^{**}$ & $\textbf{52.4}^{**}$  & $\ \ \textbf{6.6}$ & $\textbf{51.4}^{**}$ & $\textbf{55.0}^{**}$  & $\textbf{10.2}^*$\\
\hline
\end{tabular}
}\vspace{-.1cm}
\caption{Main results. $^\dag$Quoted from \citet{yang-etal-2023-unsupervised}. $^*$$p$-value $<0.05$ in an Improved Nonrandomized Sign test~\citep{INStest} against the best ensemble individual in each metric, indicated by \underline{underline}.  $^{**}$$p$-value $<0.01$.}
\label{tab:main_results}
\end{table*}

\textbf{Baselines.} Our ensemble individual, TN-LCFRS~\citep{yang-etal-2023-unsupervised}, is naturally included as a baseline. In addition, we consider different variants of continuous PCFG parsers and discontinuous LCFRS parsers, based on vanilla neural networks, compound prior~\citep{kim-etal-2019-compound}, and tensor decomposition-based neural networks~\citep{yang-etal-2021-pcfgs}, denoted by $[$N$|$C$|$TN$]$-$[$PCFG$|$LCFRS$]$.

\subsection{Results and Analyses}
\label{sec:results}

\textbf{Main Results.} Tabel~\ref{tab:main_results} presents the main results on three datasets. We replicated TN-LCFRS with five runs as our ensemble individuals; our results are similar to \citet{yang-etal-2023-unsupervised}, showing the success of our replication (Rows~9~vs.~8).

In our study, we observe TN-LCFRS behaves inconsistently in different runs: some runs are good at continuous constituents~(Row~11), while others are good at discontinuous ones~(Row~12); both may disagree with the best run according to $F_1^\text{overall}$~(Row~10).

Our ensemble approach (Rows~14) achieves $F_1^{\text{cont}}$ and $F_1^{\text{disco}}$  scores comparable to, or higher than, all individuals. This eventually leads to a much higher $F_1^{\text{overall}}$ score, with a $p$-value $<0.01$ in an Improved Nonrandomized Sign (INS) test~\citep{INStest} with the binomial comparison over sentence-level performance.\footnote{We chose INS test because it can properly handle a large number of ties (neutral cases), when they do not necessarily suggest the equivalency of the models but the inadequacy of test samples (e.g., no discontinuous constituents in a sentence).} Results are consistent on all three datasets. 

\begin{table}[t]
\centering
\resizebox{\linewidth}{!}{
\begin{tabular}{|l|*{3}{r r|}}
\hline
&
\multicolumn{2}{c|}{Overall} &
\multicolumn{2}{c|}{Cont} &
\multicolumn{2}{c|}{Disco}\\
Model & \multicolumn{1}{c}{TP} & \multicolumn{1}{c|}{FP} & \multicolumn{1}{c}{TP} & \multicolumn{1}{c|}{FP} & \multicolumn{1}{c}{TP} & \multicolumn{1}{c|}{FP}\\
  \hline
Binary ensemble & $6392$ & $9318$ & $6290$ & $7934$ & $102$ & $1384$\\ 
Non-binary ensemble & $6321$ & $8858$ & $6221$ & $7605$ & $100$ & $1253$\\
  \hline
Difference & $71$ & $460$ & $69$ & $329$ & $2$ & $131$\\
  \hline
\end{tabular}
}\vspace{-.1cm}
\caption{Number of true positive (TP) and false positive (FP) constituents on LASSY.}
\label{table:binary_vs_nonbinary}
\end{table}

\textbf{Binary vs.~  Non-Binary Ensemble.} To the best of our knowledge, we are the first to address non-binary unsupervised constituency parsing. A natural question is whether and how non-binarity in the output improves the performance, given that all individuals are binary. To this end, we conducted an experiment by restricting the output to binary trees.\footnote{For binary outputs, we may still perform the meet-in-the-middle search for non-binary trees, with the hit count as the objective, and binarize them by post-processing. Note that Theorem~\ref{thm:lowerbound} for pruning does not hold in this case, but we may still safely ignore zero-hit constituents, making the algorithm affordable (although it is not as efficient as the non-binary setting where Theorem~\ref{thm:lowerbound} can be applied).} We compare the results with the non-binary setting in Rows~13--14, Table~\ref{tab:main_results}, showing the consistent superiority of our non-binary ensemble in all metrics and datasets. To further understand the effect of non-binarity, we report in Table~\ref{table:binary_vs_nonbinary} true positive and false positive counts on the LASSY dataset.\footnote{We chose LASSY as the testbed for all our analyses due to the limit of space and time, because it is relatively more stable in terms of $F_1^\text{disco}$ than other datasets.} Results suggest that the non-binary setting largely eliminates false positive constituents, as the model may opt to predict fewer constituents than that in the binary setting.

\begin{figure}[t]
\begin{center}
\includegraphics[width=1\linewidth]{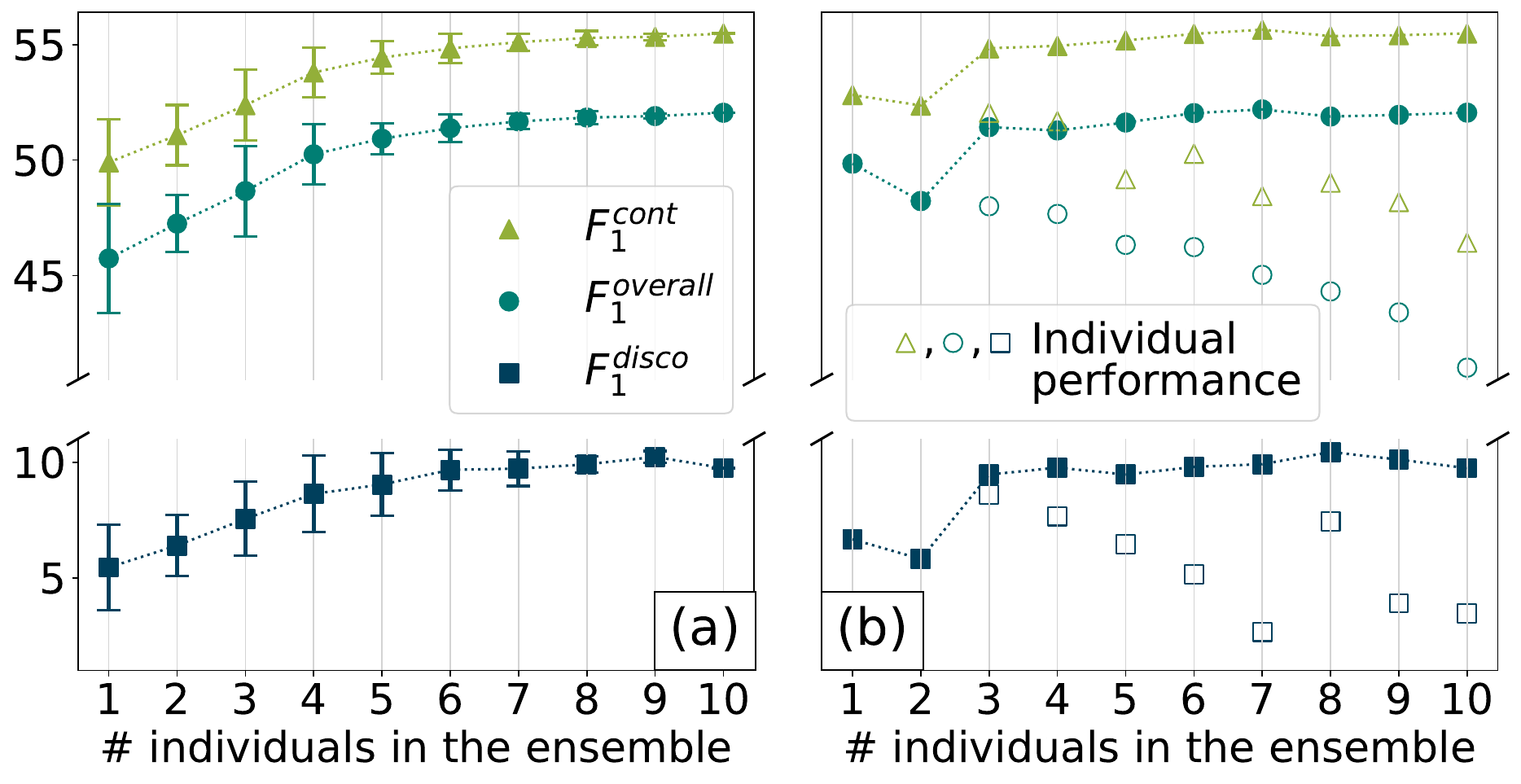}
\end{center}\vspace{-10pt}
\caption{Effect of the number of ensemble individuals on LASSY. (a) Averaged over $30$ trials with error bars indicating standard derivations. (b) Best-to-worst incremental ensemble.}
\label{fig:analysis_incremental}
\end{figure}

\textbf{Number of Ensemble Individuals.}
We varied the number of individuals in the ensemble to analyze its effect. In Figure~\ref{fig:analysis_incremental}a, we picked a random subset of individuals from a pool of $10$ and averaged the results over $30$ trials. Overall, more individuals lead to higher performance and lower variance in all $F_1$ scores, although the performance may be saturated if there are already a large number of individuals.

In addition, we present an analysis of the best-to-worst incremental ensembles in Figure~\ref{fig:analysis_incremental}b, where we add individuals to the ensemble from the best to worst based on $F_1^\text{overall}$. As seen, adding weak individuals to the ensemble does not hurt, if not help, the performance in our experiments, demonstrating that our ensemble approach is robust to the quality of individuals.

\begin{figure}[t]
\begin{center}\vspace{-.2cm}
\includegraphics[width=1\linewidth]{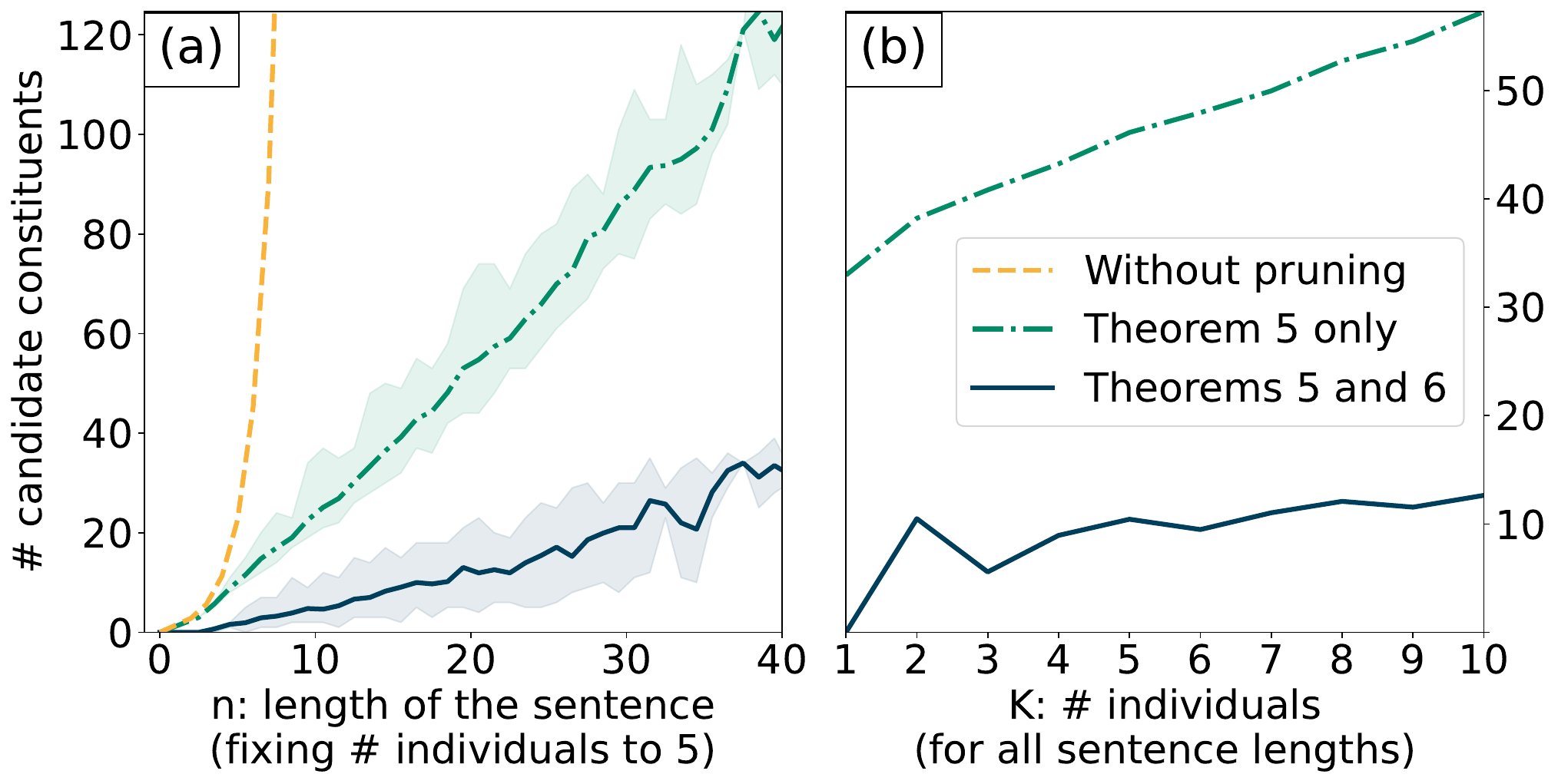}
\end{center}\vspace{-10pt}
\caption{Effectiveness of pruning on LASSY for (a) different sentence lengths, and (b) different numbers of ensemble individuals. Note that the dashed orange line does not fit the range of $y$-axis in (b).}
\label{fig:complexity_pruning}
\end{figure}

\begin{figure}[t]
\begin{center}
\includegraphics[width=1\linewidth]{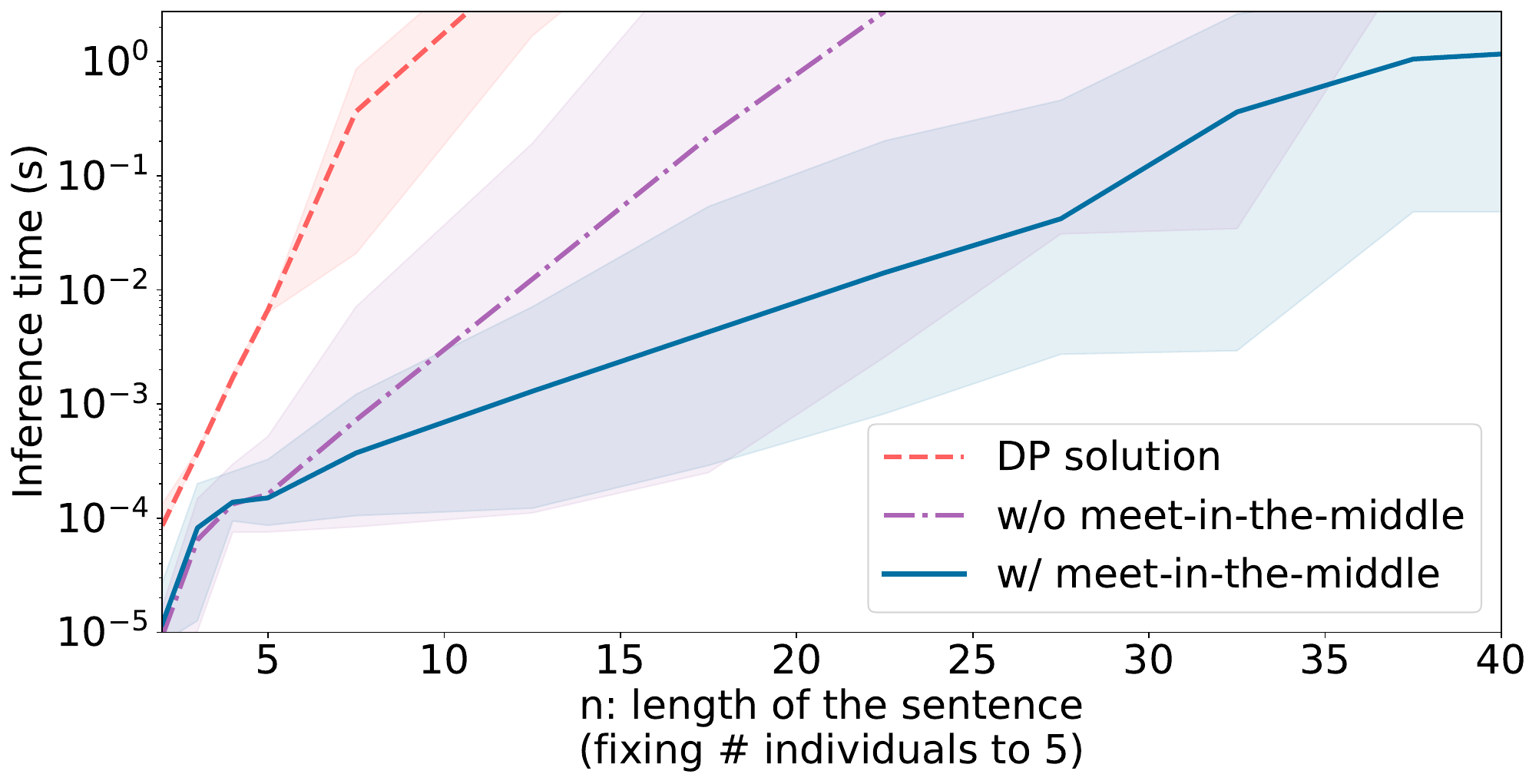}
\end{center}\vspace{-10pt}
\caption{Wall clock run time of tree-averaging algorithms on LASSY for different sentence lengths, using an Intel(R) Core(TM) i9-9940X (@3.30GHz) CPU.}
\label{fig:run_time}
\end{figure}

\textbf{Efficiency Analysis.} We analyze the effect of our pruning mechanism and the meet-in-the-middle algorithm. Pruning serves as preprocessing for our search, and we show the number of remaining candidate constituents in Figure~\ref{fig:complexity_pruning}. The empirical results confirm Theorem~\ref{thm:lowerbound} bounding the number of candidates linear in both sentence length and the number of individuals, as well as Theorem~\ref{thm:upperbound} further bounding it by a constant with respect to the number of individuals.

We further analyze wall clock run time in Figure~\ref{fig:run_time}, as the actual execution time may be different from algorithmic complexity. In \S\ref{Apndx:proof:boundeddiscontinuous}, we provide a dynamic programming (DP) algorithm with time complexity of  $\mathcal O(n^9)$ for a length-$n$ sentence. As shown in the figure, the DP fails to serve as a practical algorithm due to the high-order polynomial. By contrast, our exponential-time search algorithm, even without the meet-in-the-middle technique, is able to run in a reasonable time for many samples (especially short ones) because of the strong pruning mechanism. Our meet-in-the-middle technique further speeds up the search, making our algorithm efficient and faster than DP for all the data samples.

It is emphasized that this experiment also empirically verifies the correctness of all our algorithms, as they perform exact inference and we have obtained exact results by using different algorithms. 

\textbf{Additional Results.} We provide additional results in the appendix. \ref{apndx:per_label_analysis}: Performance by constituency types; and \ref{appndx:case}: Case study.

\section{Related Work}

Constituency parsing carries a long history in natural language processing research~\citep{charniak-2000-maximum, klein2005unsupervised, kallmeyer-maier-2010-data,li-etal-2019-imitation}. Different setups of the task have been introduced and explored, including supervised and unsupervised, continuous and discontinuous constituency parsing~\citep{shen2017neural, shen-etal-2018-straight, shen2018ordered, corro-2020-span, DiscontinuousConstituencyParsingWithPointerNetworks, fernandez-gonzalez-gomez-rodriguez-2021-reducing, FERNANDEZGONZALEZ202343, chen-komachi-2023-discontinuous, yang-etal-2023-unsupervised, yang-tu-2023-dont}. In the unsupervised setup, researchers typically define a joint distribution over the parse structure and an observable variable, e.g., the sentence itself, and maximize the observable variable's likelihood through marginalization~\citep{kim-etal-2019-compound, kim-etal-2019-unsupervised}. To the best of our knowledge, previous unsupervised parsing studies are all restricted to binary structures to squeeze the marginalization space, and we are the first to address non-binary unsupervised parsing and show non-binarity is beneficial to parsing performance.

Ensemble methods strategically combine multiple models to improve performance, rooted in the bagging concept where different data portions are used to train multiple models ~\citep{breiman1996bagging,hastie2009boosting}. To build an ensemble, straightforward methods include averaging and voting~\citep{breiman1996bagging,breiman1996heuristics}. 
For outputs with internal structures, minimum Bayes risk decoding~\citep[MBR;][]{MBRbook} can be used to build an ensemble, where the vote is the negative risk in MBR. However, existing MBR approaches are mostly \textit{selective}, where the output is selected from a candidate set~\citep{kumar-byrne-2004-minimum,titov-henderson-2006-loss,shi-etal-2022-natural}.
\citet{petrov-klein-2007-improved} formulate MBR for supervised constituency parsing and propose to search for the global optimum decoding when the risk allows for dynamic programming. \citet{smith-smith-2007-probabilistic} extend the idea to non-projective dependency parsing. In our previous work~\cite{shayegh2023ensemble}, we 
formally introduce \textit{generative} MBR and develop an algorithm that searches for a binary continuous tree. This paper extends our previous work and searches in the space of discontinuous constituency trees, leading to significant algorithmic design and theoretical analysis.

\section{Conclusion}

In this work, we address ensemble-based unsupervised discontinuous constituency parsing by tree averaging, where we provide comprehensive complexity analysis and develop an efficient search algorithm to obtain the average tree. Our experiments on Dutch and German demonstrate the effectiveness of our ensemble method. To the best of our knowledge, we are also the first to address, and show the importance of, non-binary structures in unsupervised constituency parsing. 

\section{Limitations}
Our work demonstrates both theoretical depth and empirical effectiveness, but may also have limitations. 

First, our work is focused on syntactic parsing, and we have provided a series of theoretical analyses and algorithmic designs for averaging constituency trees under different setups. Following the trajectory, our ensemble approach may be extended to other structures beyond parsing. We are happy to explore this direction as future work. For example, a concurrent study of ours develops algorithms for ensemble-based text generation~\cite{wen2024ebbs}.

Second, our model is only tested on Dutch and German datasets. This is partially because of the established setups in previous work~\citep{yang-etal-2023-unsupervised} and the lack of annotated treebanks. Notice that English is usually excluded from the study of unsupervised discontinuous parsing, because English discontinuous structures are too rare for any model to discover. A potential future direction is multilingual linguistic structure discovery, perhaps, with ensembles of different languages. 

Third, our theoretical analysis leaves an open problem about averaging binary trees. However, pointing to open problems is usually considered a contribution (instead of a weakness) in theoretical computer science. Our theoretical analysis is also crucial to the understanding of our algorithms, because we now know that our proposed method works at the level of Open Problem~\ref{openproblem:binarydisco} and can be easily extended to the level of Problem~\ref{problem:nonbinarydisco} following the technique proposed in Appendix~\ref{Apndx:proof:nonbinaryboundeddiscontinuous}. 

\section*{Acknowledgments}
We would like to thank all reviewers and chairs for their valuable and constructive comments. The research is supported in part by the Natural Sciences and Engineering Research Council of Canada (NSERC), the Amii Fellow Program, the Canada CIFAR AI Chair Program, an Alberta Innovates Program, a donation from DeepMind, and the Digital Research Alliance of Canada (alliancecan.ca). We also thank Yara Kamkar for providing advice on the algorithms.

\newpage

\bibliography{custom}

\begin{thebibliography}{55}
\expandafter\ifx\csname natexlab\endcsname\relax\def\natexlab#1{#1}\fi

\bibitem[{Arora and Barak(2009)}]{arora2009computational}
Sanjeev Arora and Boaz Barak. 2009.
\newblock \href {https://www.cambridge.org/us/catalogue/catalogue.asp?isbn=9780521424264} {\emph{Computational Complexity: A Modern Approach}}.
\newblock Cambridge University Press.

\bibitem[{Bickel and Doksum(2015)}]{MBRbook}
Peter~J Bickel and Kjell~A Doksum. 2015.
\newblock \href {https://bickel.stat.berkeley.edu/teaching} {\emph{Mathematical Statistics: Basic Ideas and Selected Topics}}.
\newblock CRC Press.

\bibitem[{Bod(2009)}]{Exemplar2Grammar}
Rens Bod. 2009.
\newblock \href {https://onlinelibrary.wiley.com/doi/10.1111/j.1551-6709.2009.01031.x} {From exemplar to grammar: A probabilistic analogy-based model of language learning}.
\newblock \emph{Cognitive Science}, 33(5):752--793.

\bibitem[{Brants et~al.(2002)Brants, Dipper, Hansen, Lezius, and Smith}]{TIGER_TB}
Sabine Brants, Stefanie Dipper, Silvia Hansen, Wolfgang Lezius, and George Smith. 2002.
\newblock \href {https://api.semanticscholar.org/CorpusID:6209052} {The {TIGER} treebank}.
\newblock In \emph{Proceedings of the Workshop on Treebanks and Linguistic Theories}, pages 24--41.

\bibitem[{Breiman(1996{\natexlab{a}})}]{breiman1996bagging}
Leo Breiman. 1996{\natexlab{a}}.
\newblock \href {https://link.springer.com/article/10.1007/BF00058655} {Bagging predictors}.
\newblock \emph{Machine Learning}, 24:123--140.

\bibitem[{Breiman(1996{\natexlab{b}})}]{breiman1996heuristics}
Leo Breiman. 1996{\natexlab{b}}.
\newblock \href {https://projecteuclid.org/journals/annals-of-statistics/volume-24/issue-6/Heuristics-of-instability-and-stabilization-in-model-selection/10.1214/aos/1032181158.full} {Heuristics of instability and stabilization in model selection}.
\newblock \emph{The Annals of Statistics}, 24(6):2350--2383.

\bibitem[{Cao et~al.(2020)Cao, Kitaev, and Klein}]{cao-etal-2020-unsupervised}
Steven Cao, Nikita Kitaev, and Dan Klein. 2020.
\newblock \href {https://aclanthology.org/2020.emnlp-main.389} {Unsupervised parsing via constituency tests}.
\newblock In \emph{Proceedings of the Conference on Empirical Methods in Natural Language Processing}, pages 4798--4808.

\bibitem[{Carnie(2007)}]{constituencyTree}
Andrew Carnie. 2007.
\newblock \href {https://www.wiley.com/en-ca/Syntax:+A+Generative+Introduction,+4th+Edition-p-9781119569237} {\emph{Syntax: A Generative Introduction}}, 2 edition.
\newblock Wiley Blackwell.

\bibitem[{Charniak(2000)}]{charniak-2000-maximum}
Eugene Charniak. 2000.
\newblock \href {https://aclanthology.org/A00-2018} {A maximum-entropy-inspired parser}.
\newblock In \emph{the North American Chapter of the Annual Meeting of the Association for Computational Linguistics}, pages 132--139.

\bibitem[{Chen and Komachi(2023)}]{chen-komachi-2023-discontinuous}
Zhousi Chen and Mamoru Komachi. 2023.
\newblock \href {https://aclanthology.org/2023.tacl-1.16} {Discontinuous combinatory constituency parsing}.
\newblock \emph{Transactions of the Association for Computational Linguistics}, pages 267--283.

\bibitem[{Clark(2001)}]{clark-2001-unsupervised}
Alexander Clark. 2001.
\newblock \href {https://aclanthology.org/W01-0713} {Unsupervised induction of stochastic context-free grammars using distributional clustering}.
\newblock In \emph{Proceedings of the Annual Meeting of the Association for Computational Linguistics Workshop on Computational Natural Language Learning}.

\bibitem[{Corro(2020)}]{corro-2020-span}
Caio Corro. 2020.
\newblock \href {https://aclanthology.org/2020.emnlp-main.219} {Span-based discontinuous constituency parsing: a family of exact chart-based algorithms with time complexities from {O}(n{\^{}}6) down to {O}(n{\^{}}3)}.
\newblock In \emph{Proceedings of the Conference on Empirical Methods in Natural Language Processing}, pages 2753--2764.

\bibitem[{Corro(2023)}]{corro-2023-dynamic}
Caio Corro. 2023.
\newblock \href {https://aclanthology.org/2023.acl-long.598} {A dynamic programming algorithm for span-based nested named-entity recognition in $o(n^2)$}.
\newblock In \emph{Proceedings of the Annual Meeting of the Association for Computational Linguistics}, pages 10712--10724.

\bibitem[{Cranenburgh et~al.(2016)Cranenburgh, Scha, and Bod}]{DOPDiscoCFT}
Andreas Cranenburgh, Remko Scha, and Rens Bod. 2016.
\newblock \href {https://jlm.ipipan.waw.pl/index.php/JLM/article/view/100} {Data-oriented parsing with discontinuous constituents and function tags}.
\newblock \emph{Journal of Language Modelling}, 4(1):57--111.

\bibitem[{Fern{\'a}ndez-Gonz{\'a}lez and G{\'o}mez-Rodr{\'\i}guez(2020)}]{DiscontinuousConstituencyParsingWithPointerNetworks}
Daniel Fern{\'a}ndez-Gonz{\'a}lez and Carlos G{\'o}mez-Rodr{\'\i}guez. 2020.
\newblock \href {https://ojs.aaai.org/index.php/AAAI/article/view/6275} {Discontinuous constituent parsing with pointer networks}.
\newblock In \emph{Proceedings of the AAAI Conference on Artificial Intelligence}, pages 7724--7731.

\bibitem[{Fern{\'a}ndez-Gonz{\'a}lez and G{\'o}mez-Rodr{\'\i}guez(2021)}]{fernandez-gonzalez-gomez-rodriguez-2021-reducing}
Daniel Fern{\'a}ndez-Gonz{\'a}lez and Carlos G{\'o}mez-Rodr{\'\i}guez. 2021.
\newblock \href {https://aclanthology.org/2021.emnlp-main.825} {Reducing discontinuous to continuous parsing with pointer network reordering}.
\newblock In \emph{Proceedings of the Conference on Empirical Methods in Natural Language Processing}, pages 10570--10578.

\bibitem[{Fern{\'a}ndez-Gonz{\'a}lez and G{\'o}mez-Rodr{\'\i}guez(2023)}]{FERNANDEZGONZALEZ202343}
Daniel Fern{\'a}ndez-Gonz{\'a}lez and Carlos G{\'o}mez-Rodr{\'\i}guez. 2023.
\newblock \href {https://www.sciencedirect.com/science/article/pii/S092523122201551X?via%3Dihub} {Discontinuous grammar as a foreign language}.
\newblock \emph{Neurocomputing}, 524:43--58.

\bibitem[{Goldsmith(2001)}]{10.1162/089120101750300490}
John Goldsmith. 2001.
\newblock \href {https://direct.mit.edu/coli/article/27/2/153/1711/Unsupervised-Learning-of-the-Morphology-of-a} {Unsupervised learning of the morphology of a natural language}.
\newblock \emph{Computational Linguistics}, 27(2):153–198.

\bibitem[{Hastie et~al.(2009)Hastie, Tibshirani, Friedman, Hastie, Tibshirani, and Friedman}]{hastie2009boosting}
Trevor Hastie, Robert Tibshirani, Jerome Friedman, Trevor Hastie, Robert Tibshirani, and Jerome Friedman. 2009.
\newblock \href {https://link.springer.com/content/pdf/10.1007/978-0-387-84858-7_10.pdf} {\emph{The Elements of Statistical Learning: Data Mining, Inference, and Prediction}}.
\newblock Springer.

\bibitem[{Horowitz and Sahni(1974)}]{horowitz1974computing}
Ellis Horowitz and Sartaj Sahni. 1974.
\newblock \href {https://dl.acm.org/doi/10.1145/321812.321823} {Computing partitions with applications to the knapsack problem}.
\newblock \emph{Journal of the Association for Computing Machinery}, 21(2):277–292.

\bibitem[{Joshi(1985)}]{joshi_1985}
Aravind~K Joshi. 1985.
\newblock \href {https://www.cambridge.org/core/books/natural-language-parsing/tree-adjoining-grammars-how-much-contextsensitivity-is-required-to-provide-reasonable-structural-descriptions/81BFD6DAC6B0CB24A3042A06E964F2E1} {Tree adjoining grammars: How much context-sensitivity is required to provide reasonable structural descriptions?}
\newblock In \emph{Natural Language Parsing}, pages 206--250.

\bibitem[{Kallmeyer and Maier(2010)}]{kallmeyer-maier-2010-data}
Laura Kallmeyer and Wolfgang Maier. 2010.
\newblock \href {https://aclanthology.org/C10-1061} {Data-driven parsing with probabilistic linear context-free rewriting systems}.
\newblock In \emph{Proceedings of the International Conference on Computational Linguistics}, pages 537--545.

\bibitem[{Kann et~al.(2019)Kann, Mohananey, Bowman, and Cho}]{kann-etal-2019-neural}
Katharina Kann, Anhad Mohananey, Samuel~R. Bowman, and Kyunghyun Cho. 2019.
\newblock \href {https://aclanthology.org/D19-6123} {Neural unsupervised parsing beyond {E}nglish}.
\newblock In \emph{Proceedings of the Workshop on Deep Learning Approaches for Low-Resource Natural Language Processing}, pages 209--218.

\bibitem[{Kim et~al.(2019{\natexlab{a}})Kim, Dyer, and Rush}]{kim-etal-2019-compound}
Yoon Kim, Chris Dyer, and Alexander Rush. 2019{\natexlab{a}}.
\newblock \href {https://aclanthology.org/P19-1228} {Compound probabilistic context-free grammars for grammar induction}.
\newblock In \emph{Proceedings of the Annual Meeting of the Association for Computational Linguistics}, pages 2369--2385.

\bibitem[{Kim et~al.(2019{\natexlab{b}})Kim, Rush, Yu, Kuncoro, Dyer, and Melis}]{kim-etal-2019-unsupervised}
Yoon Kim, Alexander Rush, Lei Yu, Adhiguna Kuncoro, Chris Dyer, and G{\'a}bor Melis. 2019{\natexlab{b}}.
\newblock \href {https://aclanthology.org/N19-1114} {Unsupervised recurrent neural network grammars}.
\newblock In \emph{Proceedings of the North {A}merican Chapter of the Annual Meeting of the Association for Computational Linguistics: Human Language Technologies}, pages 1105--1117.

\bibitem[{Klein(2005)}]{klein2005unsupervised}
Dan Klein. 2005.
\newblock \href {https://www.proquest.com/openview/4a4c32a46d9fab258d8114cede8311ed/1?pq-origsite=gscholar&cbl=18750&diss=y} {\emph{The Unsupervised Learning of Natural Language Structure}}.
\newblock Stanford University.

\bibitem[{Klein and Manning(2002)}]{klein-manning-2002-generative}
Dan Klein and Christopher~D. Manning. 2002.
\newblock \href {https://dl.acm.org/doi/10.3115/1073083.1073106} {A generative constituent-context model for improved grammar induction}.
\newblock In \emph{Proceedings of the Annual Meeting of the Association for Computational Linguistics}, page 128–135.

\bibitem[{Kumar and Byrne(2004)}]{kumar-byrne-2004-minimum}
Shankar Kumar and William Byrne. 2004.
\newblock \href {https://aclanthology.org/N04-1022} {Minimum {B}ayes-risk decoding for statistical machine translation}.
\newblock In \emph{Proceedings of the Human Language Technology Conference of the North {A}merican Chapter of the Association for Computational Linguistics}, pages 169--176.

\bibitem[{Li et~al.(2019)Li, Mou, and Keller}]{li-etal-2019-imitation}
Bowen Li, Lili Mou, and Frank Keller. 2019.
\newblock \href {https://aclanthology.org/P19-1338} {An imitation learning approach to unsupervised parsing}.
\newblock In \emph{Proceedings of the Annual Meeting of the Association for Computational Linguistics}, pages 3485--3492.

\bibitem[{Li and Lu(2023)}]{li-lu-2023-contextual}
Jiaxi Li and Wei Lu. 2023.
\newblock \href {https://aclanthology.org/2023.acl-long.285} {Contextual distortion reveals constituency: Masked language models are implicit parsers}.
\newblock In \emph{Proceedings of the Annual Meeting of the Association for Computational Linguistics}, pages 5208--5222.

\bibitem[{Maier(2010)}]{maier-2010-direct}
Wolfgang Maier. 2010.
\newblock \href {https://aclanthology.org/W10-1407} {Direct parsing of discontinuous constituents in {G}erman}.
\newblock In \emph{Proceedings of the NAACL HLT Workshop on Statistical Parsing of Morphologically-Rich Languages}, pages 58--66.

\bibitem[{Maier et~al.(2012)Maier, Kaeshammer, and Kallmeyer}]{maier-etal-2012-plcfrs}
Wolfgang Maier, Miriam Kaeshammer, and Laura Kallmeyer. 2012.
\newblock \href {https://aclanthology.org/W12-4615} {{PLCFRS} parsing revisited: Restricting the fan-out to two}.
\newblock In \emph{Proceedings of the International Workshop on Tree Adjoining Grammars and Related Formalisms}, pages 126--134.

\bibitem[{McCawley(1982)}]{Parentheticals1982}
James~D. McCawley. 1982.
\newblock \href {http://www.jstor.org/stable/4178261} {Parentheticals and discontinuous constituent structure}.
\newblock \emph{Linguistic Inquiry}, 13(1):91--106.

\bibitem[{Peng et~al.(2011)Peng, Wu, Zhu, and Zhang}]{6137337}
Huan-Kai Peng, Pang Wu, Jiang Zhu, and Joy~Ying Zhang. 2011.
\newblock \href {https://ieeexplore.ieee.org/abstract/document/6137337} {Helix: Unsupervised grammar induction for structured activity recognition}.
\newblock In \emph{Proceedings of the International Conference on Data Mining}, pages 1194--1199.

\bibitem[{Petrov and Klein(2007)}]{petrov-klein-2007-improved}
Slav Petrov and Dan Klein. 2007.
\newblock \href {https://aclanthology.org/N07-1051} {Improved inference for unlexicalized parsing}.
\newblock In \emph{Proceedings of Human Language Technologies: The Conference of the North {A}merican Chapter of the Association for Computational Linguistics}, pages 404--411.

\bibitem[{Shayegh et~al.(2024)Shayegh, Cao, Zhu, Cheung, and Mou}]{shayegh2023ensemble}
Behzad Shayegh, Yanshuai Cao, Xiaodan Zhu, Jackie~CK Cheung, and Lili Mou. 2024.
\newblock \href {https://openreview.net/forum?id=RR8y0WKrFv} {Ensemble distillation for unsupervised constituency parsing}.
\newblock In \emph{International Conference on Learning Representations}.

\bibitem[{Shen et~al.(2018{\natexlab{a}})Shen, Lin, Huang, and Courville}]{shen2017neural}
Yikang Shen, Zhouhan Lin, Chin-Wei Huang, and Aaron Courville. 2018{\natexlab{a}}.
\newblock \href {https://openreview.net/forum?id=rkgOLb-0W} {Neural language modeling by jointly learning syntax and lexicon}.
\newblock In \emph{International Conference on Representation Learning}.

\bibitem[{Shen et~al.(2018{\natexlab{b}})Shen, Lin, Jacob, Sordoni, Courville, and Bengio}]{shen-etal-2018-straight}
Yikang Shen, Zhouhan Lin, Athul~Paul Jacob, Alessandro Sordoni, Aaron Courville, and Yoshua Bengio. 2018{\natexlab{b}}.
\newblock \href {https://aclanthology.org/P18-1108} {Straight to the tree: Constituency parsing with neural syntactic distance}.
\newblock In \emph{Proceedings of the Annual Meeting of the Association for Computational Linguistics}, pages 1171--1180.

\bibitem[{Shen et~al.(2019)Shen, Tan, Sordoni, and Courville}]{shen2018ordered}
Yikang Shen, Shawn Tan, Alessandro Sordoni, and Aaron Courville. 2019.
\newblock \href {https://openreview.net/forum?id=B1l6qiR5F7} {Ordered neurons: Integrating tree structures into recurrent neural networks}.
\newblock In \emph{International Conference on Representation Learning}.

\bibitem[{Shi et~al.(2022)Shi, Fried, Ghazvininejad, Zettlemoyer, and Wang}]{shi-etal-2022-natural}
Freda Shi, Daniel Fried, Marjan Ghazvininejad, Luke Zettlemoyer, and Sida~I. Wang. 2022.
\newblock \href {https://aclanthology.org/2022.emnlp-main.231} {Natural language to code translation with execution}.
\newblock In \emph{Proceedings of the Conference on Empirical Methods in Natural Language Processing}, pages 3533--3546.

\bibitem[{Skut et~al.(1997)Skut, Krenn, Brants, and Uszkoreit}]{skut-etal-1997-annotation}
Wojciech Skut, Brigitte Krenn, Thorsten Brants, and Hans Uszkoreit. 1997.
\newblock \href {https://dl.acm.org/doi/10.3115/974557.974571} {An annotation scheme for free word order languages}.
\newblock In \emph{Proceedings of the Applied Natural Language Processing}, page 88–95.

\bibitem[{Smith and Smith(2007)}]{smith-smith-2007-probabilistic}
David~A. Smith and Noah~A. Smith. 2007.
\newblock \href {https://aclanthology.org/D07-1014} {Probabilistic models of nonprojective dependency trees}.
\newblock In \emph{Proceedings of the Joint Conference on Empirical Methods in Natural Language Processing and Computational Natural Language Learning}, pages 132--140.

\bibitem[{Snyder et~al.(2009)Snyder, Naseem, and Barzilay}]{snyder-etal-2009-unsupervised}
Benjamin Snyder, Tahira Naseem, and Regina Barzilay. 2009.
\newblock \href {https://aclanthology.org/P09-1009} {Unsupervised multilingual grammar induction}.
\newblock In \emph{Proceedings of the Joint Conference of the Annual Meeting of the Association for Computational Linguistics and the International Joint Conference on Natural Language Processing of the Asian Federation of Natural Language Processing}, pages 73--81.

\bibitem[{Stanojevi{\'c} and Steedman(2020)}]{stanojevic-steedman-2020-span}
Milo{\v{s}} Stanojevi{\'c} and Mark Steedman. 2020.
\newblock \href {https://aclanthology.org/2020.iwpt-1.12} {Span-based {LCFRS}-2 parsing}.
\newblock In \emph{Proceedings of the International Conference on Parsing Technologies and the International Workshop on Parsing Technologies Shared Task on Parsing into Enhanced Universal Dependencies}, pages 111--121.

\bibitem[{Starks(1979)}]{INStest}
Thomas~H Starks. 1979.
\newblock \href {http://www.jstor.org/stable/1268292} {An improved sign test for experiments in which neutral responses are possible}.
\newblock \emph{Technometrics}, 21(4):525--530.

\bibitem[{Titov and Henderson(2006)}]{titov-henderson-2006-loss}
Ivan Titov and James Henderson. 2006.
\newblock \href {https://aclanthology.org/W06-1666} {Loss minimization in parse reranking}.
\newblock In \emph{Proceedings of the Conference on Empirical Methods in Natural Language Processing}, pages 560--567.

\bibitem[{Tomita(1990)}]{10.5555/532941}
Masaru Tomita. 1990.
\newblock \href {https://link.springer.com/book/10.1007/978-1-4615-3986-5} {\emph{Current Issues in Parsing Technology}}.
\newblock Kluwer Academic Publishers.

\bibitem[{Van~Noord et~al.(2013)Van~Noord, Bouma, Van~Eynde, De~Kok, Van~der Linde, Schuurman, Sang, and Vandeghinste}]{vanNoord2013}
Gertjan Van~Noord, Gosse Bouma, Frank Van~Eynde, Daniel De~Kok, Jelmer Van~der Linde, Ineke Schuurman, Erik Tjong~Kim Sang, and Vincent Vandeghinste. 2013.
\newblock \href {https://link.springer.com/chapter/10.1007/978-3-642-30910-6_9} {Large scale syntactic annotation of written dutch: Lassy}.
\newblock In \emph{Essential Speech and Language Technology for Dutch}, pages 147--164.

\bibitem[{Versley(2014)}]{versley-2014-experiments}
Yannick Versley. 2014.
\newblock \href {https://aclanthology.org/W14-6104} {Experiments with easy-first nonprojective constituent parsing}.
\newblock In \emph{Proceedings of the Joint Workshop on Statistical Parsing of Morphologically Rich Languages and Syntactic Analysis of Non-Canonical Languages}, pages 39--53.

\bibitem[{Vijay-Shanker et~al.(1987)Vijay-Shanker, Weir, and Joshi}]{vijay-shanker-etal-1987-characterizing}
K.~Vijay-Shanker, David~J. Weir, and Aravind~K. Joshi. 1987.
\newblock \href {https://dl.acm.org/doi/10.3115/981175.981190} {Characterizing structural descriptions produced by various grammatical formalisms}.
\newblock In \emph{Proceedings of the Annual Meeting of the Association for Computational Linguistics}, page 104–111.

\bibitem[{Wen et~al.(2024)Wen, Shayegh, Huang, Cao, and Mou}]{wen2024ebbs}
Yuqiao Wen, Behzad Shayegh, Chenyang Huang, Yanshuai Cao, and Lili Mou. 2024.
\newblock \href {https://arxiv.org/abs/2403.00144} {{EBBS}: An ensemble with bi-level beam search for zero-shot machine translation}.
\newblock \emph{arXiv preprint arXiv:2403.00144}.

\bibitem[{Yang et~al.(2023)Yang, Levy, and Kim}]{yang-etal-2023-unsupervised}
Songlin Yang, Roger Levy, and Yoon Kim. 2023.
\newblock \href {https://aclanthology.org/2023.acl-long.316} {Unsupervised discontinuous constituency parsing with mildly context-sensitive grammars}.
\newblock In \emph{Proceedings of the Annual Meeting of the Association for Computational Linguistics}, pages 5747--5766.

\bibitem[{Yang and Tu(2023)}]{yang-tu-2023-dont}
Songlin Yang and Kewei Tu. 2023.
\newblock \href {https://aclanthology.org/2023.acl-long.469} {Don{'}t parse, choose spans! continuous and discontinuous constituency parsing via autoregressive span selection}.
\newblock In \emph{Proceedings of the Annual Meeting of the Association for Computational Linguistics}, pages 8420--8433.

\bibitem[{Yang et~al.(2021)Yang, Zhao, and Tu}]{yang-etal-2021-pcfgs}
Songlin Yang, Yanpeng Zhao, and Kewei Tu. 2021.
\newblock \href {https://aclanthology.org/2021.naacl-main.117} {{PCFG}s can do better: Inducing probabilistic context-free grammars with many symbols}.
\newblock In \emph{Proceedings of the North American Chapter of the Association for Computational Linguistics: Human Language Technologies}, pages 1487--1498.

\bibitem[{Younger(1967)}]{YOUNGER1967189}
Daniel~H. Younger. 1967.
\newblock \href {https://www.sciencedirect.com/science/article/pii/S001999586780007X?via%3Dihub} {Recognition and parsing of context-free languages in time $n^3$}.
\newblock \emph{Information and Control}, 10(2):189--208.

\end{thebibliography}

\appendix

\newpage
\onecolumn
\section{Proofs}
\label{Apndx:proofs}

\subsection{Proof of Theorem~\ref{thm:boundeddiscontinuous}}
\label{Apndx:proof:boundeddiscontinuous}

\problemboundeddiscontinuous*

\theoremboundeddiscontinuous*

\begin{proof}
Following our previous work~\citep{shayegh2023ensemble}, we may design a DP table with two axes being the start and end of every component of a constituent, thus $2F$-axes for fan-out~$F$. For outputs being non-binary, our DP algorithm requires one additional axis to indicate the number of nodes in a subtree.\footnote{If the output is restricted to binary trees, we do not need to track the number of nodes in a subtree, because it is always $2n-1$ for $n$ leaves.}

We define a DP variable
$H(c,\tau)$ as the best total hit count for a $\tau$-node  constituency substree over a constituent $c$, which may be discontinuous. 
A  constituent with fan-out $f$ has $f$ components, each being a span of consecutive words. Therefore, $c$ can be represented by $2F$ numbers indicating the beginnings and ends of the components, as the fan-out is bounded by $F$. For $\tau$, it is upper-bounded by $2n-1$ for a length-$n$ sentence. Overall, the DP table has a size of $\mathcal{O} (n^{2F+1})$. Specially, we further define  $H(c, 0)=0$ for every $c$.

For initialization, we consider every single-word constituent $c$ and set $H(c, 1) = \mathpzc h(c)$, where $\mathpzc h(c)$ is the hit count of $c$ in $T_1, \cdots, T_K$. We also set $H(c, i)= -\infty$ for every $i>1$.

For recursion, we divide a constituent $c$ into smaller sub-constituents based on alternation points $\bm j=(0, j_1, j_2, \cdots, j_{2F-1}, |c|)$, where $|c|$ is the number of words in $c$, shown in Figure~\ref{fig:theorem1}. We further join sub-constituents based on the parity (even or odd) of their indexes. 
\begin{align}
    o_{\bm j} &= \bigoplus_{i=1,3,\cdots, 2F-1} c[j_{i-1}:j_i]\\
    e_{\bm j} &= \bigoplus_{i=2,4,\cdots, 2F} c[j_{i-1}:j_i]
\end{align}
where $c[b:e]$ denotes a sub-constituent containing $b$th to $e$th words in $c$, and $\oplus$ symbol is the joint of sub-constituents. We also distribute the capacity of the number of nodes into $e_{\bm j}$ and $o_{\bm j}$ branches by setting $s$ as the $o_{\bm j}$'s share and assigning the rest to $e_{\bm j}$.

To find the best alternation points and shares, we iterate over all possible values for them:
\begin{align}
    \bm j^{*}_\tau, s^*_\tau = 
    \operatorname*{argmax}_{\bm j, \;\; 0 \leq \mathpzc s \leq \tau}\;
    \Big[H(o_{\bm j}, s) + H(e_{\bm j}, \tau-s)\Big]
\end{align}
where $\bm j=(0, j_1, \cdots, j_{2F-1}, |c|)$ must satisfy $j_{i} \leq j_{i+1}\le |c|$, $j_1\ne |c|$, and if $j_{i}\ne |c|$, then $j_{i} < j_{i+1}$. We discard the cases that $o_{\bm j}$ or $e_{\bm j}$ have a fan-out greater than $F$, because they are guaranteed not to appear in the output (Theorem~\ref{thm:lowerbound}, whose proof does not rely on this theorem).

Our tree-building search processing first assumes a node is binary and then decides whether to join its parent and children to achieve non-binarity. Therefore, we calculate the best hit count when $c$ itself is excluded, denoted by $H_\text{excl}$, or included, denoted by~$H_\text{incl}$. We have
\begin{align}
    H_\text{excl}(c, \tau) &= H(o_{\bm j^{*}_\tau}, s^*_\tau) + H(e_{\bm j^{*}_\tau}, \tau-s^*_\tau)\\
    H_\text{incl}(c, \tau) &=  H_\text{excl}(c, \tau-1) + \mathpzc h(c)
\end{align}
Finally, the recursion is 
\begin{align}
    H(c,\tau) &= \max\{H_\text{excl}(c, \tau), H_\text{incl}(c, \tau)\}
\end{align}

The best score of $\sum_{k=1}^K F_1(T, T_k)$ for the sentence $S$ is
\begin{align}
\operatorname*{argmax}_{n < \tau < 2n} \frac{H(S, \tau)}{\tau+2n-1}
    \label{eq:DPscore}
\end{align}
according to Eqn.~\eqref{eq:SimpliedAvg}. Problem~\ref{problem:boundeddiscontinuous}, as a decision problem, can then be solved trivially by checking whether the score in \eqref{eq:DPscore} is greater than or equal to~$z$. 

Overall, the time complexity for a recursion step is $\mathcal{O}(n^{2F})$ to find $\bm j^{*}_\tau$ and  $s^*_\tau$. With an $\mathcal{O}(n^{2F+1})$-sized DP table, the complexity of the entire DP algorithm is $\mathcal O(n^{4F+1})$, which is polynomial. Therefore,  Problem~\ref{problem:boundeddiscontinuous} belongs to~\textbf{P}.
\end{proof}
The above proof concerns the decision problem (whether the score reaches or exceeds a threshold). 
To search for the average constituency tree, we may backtrack the best corresponding constituency subtrees during the recursion.

Notice that restricting output to be binary tree simplifies the algorithm by always setting $s_\tau^* = 2|o_{\bm{j}}|-1$, $H(c, \tau)=H_{\text{incl}}(c, \tau)$, and only considering the computation of $H(c, \tau)$ if $\tau = 2|c|-1$. It reduces the complexity of a recursion step to $\mathcal{O}(n^{2F-1})$ and that of the DP-table size to $\mathcal{O}(n^{2F})$, resulting in an overall complexity of $\mathcal{O}(n^{4F-1})$.

The soundness of the DP can be cross-validated by our meet-in-the-middle search (\S\ref{sec:our_search_algo} and~\S\ref{apndx:meetinthemiddle}). As both are exact algorithms, they should output exactly the same results. This is what we observed in our experiments, providing strong empirical evidence that both of our algorithms are sound.  

\begin{figure}[t]
\begin{center}
\includegraphics[width=.6\linewidth]{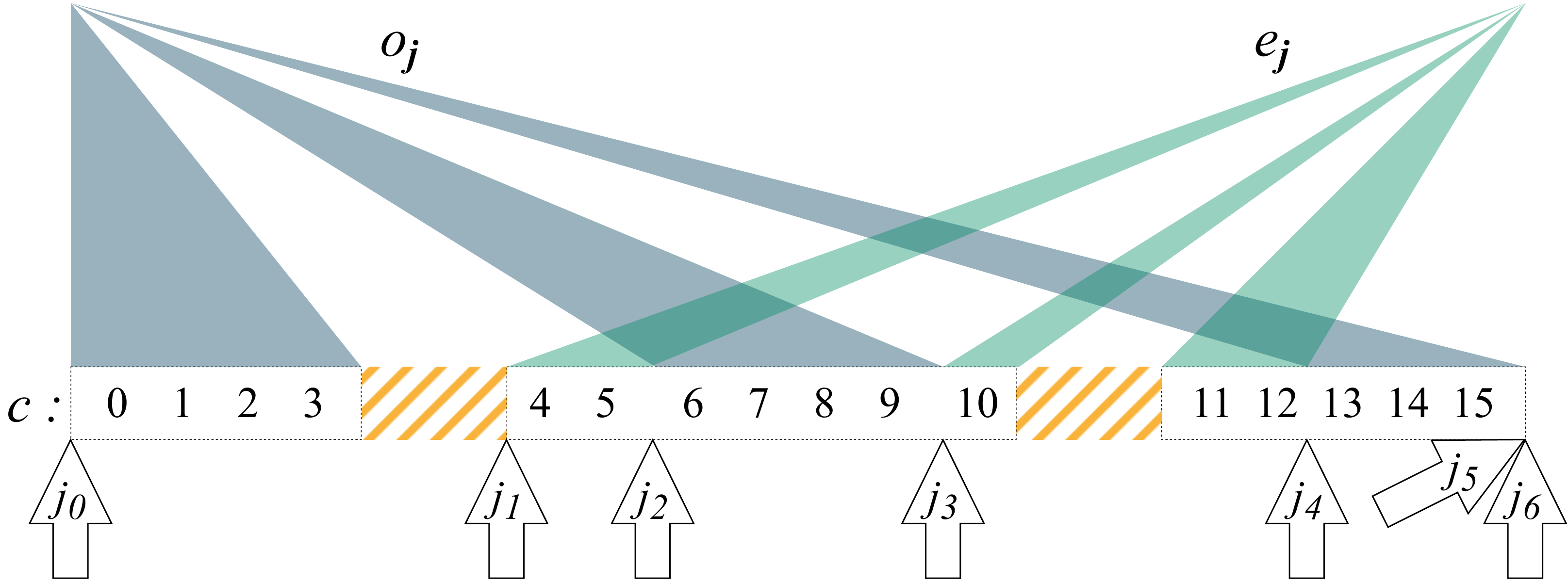}
\end{center}\vspace{-10pt}
\caption{An example of the recursion step. The yellow hachures indicate the discontinuity in the constituent $c$.}
\label{fig:theorem1}\vspace{5pt}
\end{figure}

\subsection{Proof of Theorem~\ref{thm:nonbinaryboundeddiscontinuous}}
\label{Apndx:proof:nonbinaryboundeddiscontinuous}

\problemnonbinaryboundeddiscontinuous*

\theoremnonbinaryboundeddiscontinuous*

\begin{proof}
    In \S\ref{Apndx:proof:boundeddiscontinuous}, we provide a DP algorithm and by backtracking we can easily obtain, for any $\tau$, the constituency tree $T$ that maximizes $\sum_{c\in C(T)} \mathpzc h(c)$ such that $|C(T)|=\tau$ (i.e., the tree having $\tau$ nodes). The DP algorithm, in fact, works for any scoring function $\mathpzc h$ over constituents. Therefore, we have a polynomial-time solver for the following problem:

\bigskip
    \begin{adjustwidth}{\parindent}{}
    \begin{it}
        \textbf{Fixed-Size Maximization.} Given a scoring function $\mathpzc h$ over constituents of a length-$n$ sentence and a natural number $\tau$ such that $n < \tau < 2n$, what is the tree $T$ that maximizes $\sum_{c\in C(T)} \mathpzc h(c)$ with $|C(T)|=\tau$?
    \end{it}
    \end{adjustwidth}
\bigskip

For Problem~\ref{problem:nonbinaryboundeddiscontinuous}, we may enumerate all possible values of $\tau$, i.e., $n < \tau < 2n$. Given a fixed $\tau$, we have
\begin{align}
     T^{(\tau)} &= \operatorname*{argmax}_{T: |T|=\tau} \sum_{k=1}^K F_1(T, T_k)\label{eq:maximf1fixedtau}\\
    &= \operatorname*{argmax}_{T: |T|=\tau} \sum_{k=1}^K \frac{|C(T)\cap C(T_k)|}{|C(T)|+|C(T_k)|}\\
    &= \operatorname*{argmax}_{T: |T|=\tau} \hspace{-5pt} \sum_{c\in C(T)} \underbrace{\sum_{k=1}^K \frac{\mathds{1}[c \in C(T_k)]}{\tau+|C(T_k)|}}_{\hat{\mathpzc h}(c)}\\
    &=\operatorname*{argmax}_{T: |T|=\tau} \sum_{c\in C(T)}\hat{\mathpzc h}(c)\label{eq:weightedhitcount}
\end{align}
In Eqn.~\eqref{eq:weightedhitcount}, we define a generalized scoring function $\hat{\mathpzc h}$, which can be thought of as a weighted hit count.
In other words, the overall-$F_1$ maximization in Eqn.~\eqref{eq:maximf1fixedtau} can be reduced to the above fixed-size maximization problem where the number of nodes is given, which can be solved in polynomial time. 
We may repeatedly solve the problem for $(n-2)$-many values of $\tau$ and find the best answer among them:
\begin{align}
    \tau^{*} &= \operatorname*{argmax}_{n < \tau < {2n}} \quad \sum_{k=1}^K F_1(T^{(\tau)}, T_k)\\
    T^* &= T^{(\tau^*)}
\end{align}
This does not push the complexity beyond polynomial. Finally, we can answer Problem~\ref{problem:nonbinaryboundeddiscontinuous} by checking whether $\sum_{k=1}^K F_1(T^*, T_k) \geq z$.
\end{proof}

\subsection{Proof of Theorem~\ref{thm:nonbinarydisco}}
\label{Apndx:proof:nonbinarydisco}

\problemnonbinarydisco*

\theoremnonbinarydisco*

\begin{proof} To show a problem is \textbf{NP}-complete, we need to first show it is an \textbf{NP} problem, i.e., polynomial-time solvable with a non-deterministic Turing machine. Then, we need to show its completeness, i.e., any \textbf{NP} problem, or a known \textbf{NP}-complete problem, can be reduced to this problem in polynomial time.

[Being \textbf{NP}] Being non-deterministic polynomial-time solvable is equivalent to the ability to be verified with a \textit{certificate} in polynomial time~\citep{arora2009computational}. In our case, verifying the score of any candidate tree (which may serve as the certificate) can be done in polynomial time, proving the \textbf{NP} part.

[Being complete]
To show the completeness, it suffices to reduce a known \textbf{NP}-complete problem to the problem at hand. In particular, we would reduce the max clique problem, known to be \textbf{NP}-complete~\citep{arora2009computational}, to our problem.

\bigskip
\begin{adjustwidth}{\parindent}{}
\begin{it}
    \textbf{Max Clique.}
    \label{problem:maxclique} Consider a number $z$ and a graph~$G$. Is there a clique in $G$ having at least $z$ vertices?
\end{it}
\end{adjustwidth}
\bigskip

For reduction, we use a solver for Problem~\ref{problem:nonbinarydisco} to solve the max clique problem with $G=\langle V, E\rangle$ and $z$ being given. 

Without loss of generality, we may assume that  $G$ does not contain any vertex that is connected to all other vertices for the rest of this proof. This is because such a vertex is guaranteed to appear in any maximal clique and can be pruned. Then the original problem is equivalent to determining whether the pruned graph has a clique with $(z-a)$-many vertices, where $a$ is the number of pruned vertices. Notice that such pruning can be accomplished in polynomial time, thus not changing the complexity category.

\begin{figure}[t]
\begin{center}
\includegraphics[width=.5\linewidth]{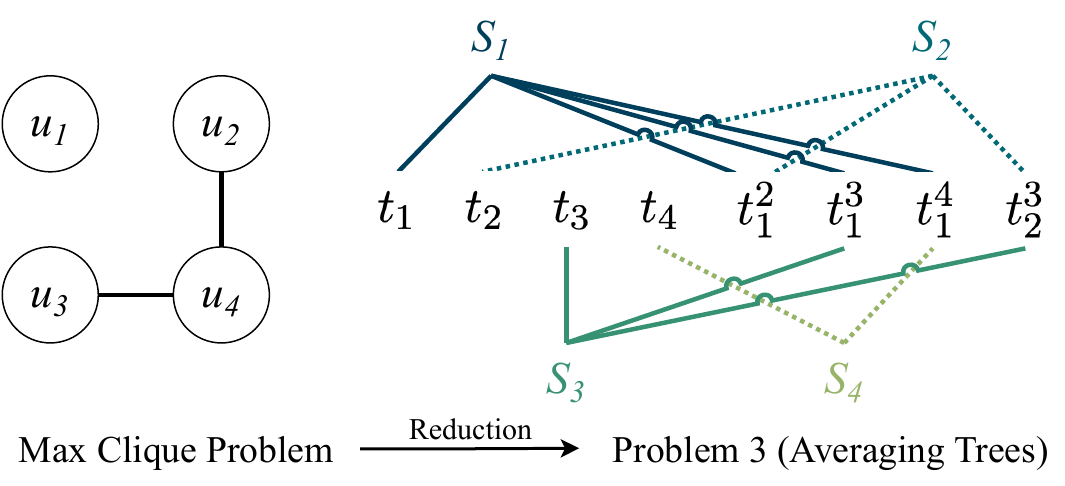}
\end{center}\vspace{-10pt}
\caption{An example of reduction.}\vspace{10pt}
\label{fig:theorem3}
\end{figure}

We construct an instantiation of Problem~\ref{problem:nonbinarydisco} as follows.
For every $u_i \in {V}$, construct a set $S_i$ initialized by a special symbol $t_i$. For every disconnected vertex pair $(u_i, u_j)$ such that $i< j$, place a special symbol $t_i^j$ in both $S_i$ and $S_j$, shown in Figure~\ref{fig:theorem3}. We can view each $t_i$ and $t_i^j$ as words in a sentence. Each $S_i$ represents a possible constituent containing the words corresponding to its elements.

For each $S_i$, construct a constituency tree $T_i$ by all the single words, the whole-sentence constituent, and $S_i$.
Note that $S_i$ cannot be a single word because there is no vertex in $G$ connected to all other vertices. Moreover, $S_i$ cannot be the whole-sentence constituent either, because $S_i$ does not contain the word $t_{i'}$ for $i'\ne i$. Therefore, $|C(T_i)|=n+2$ where $n$ is the number of words.

We next show that Problem~\ref{problem:nonbinarydisco} given $\{T_i\}_{i:u_i \in {V}}$ and the value $\tilde z = \frac{2{z} + 2|{V}|(n+1)}{{z}+2n+3}$ is equivalent to the max clique problem given the graph $G$ and the number ${z}$.

For every $u_i, u_j \in {V}$ with $i<j$, $u_i$ and $u_j$ are connected if and only if $S_i\cap S_j = \emptyset$; otherwise $t_i^j$ is in both $S_i$ and $S_j$. In addition, $S_i \not\subseteq S_j$ for $i\ne j$ because there is a special symbol $t_i$ in each $S_i$. In other words, $S_i$ and $S_j$ are compatible constituents if and only if $u_i$ and $u_j$ are connected in $G$. As a result, every constituency tree---constructed by single words, the whole-sentence constituent, and a subset of $\{S_i\}_{i:u_i \in {V}}$---corresponds to a clique in~$G$. We refer to such a tree as a \textit{clique tree}.

We can further find the correspondence of the scores between Problem~\ref{problem:nonbinarydisco} and the max clique problem. For every clique tree $T$, we have
\begin{align}
\sum_{i:\,u_i \in {V}} F_1({T}, {T}_i) &= \sum_{i:\,u_i \in {V}} \frac{2 (\mathds{1}[C({T}_i) \subseteq C({T})]+n+1)}{|C({T})|+n+2}\\
&= \frac{2(|C(T)|-n-1) + 2|{V}|(n+1)}{|C({T})|+n+2}
\end{align}
where $|C({T})|-n-1$ is the same as the number of vertices in the corresponding clique. This shows a clique tree ${T}$ with a score of $\tilde z = \frac{2{z} + 2|{V}|(n+1)}{{z}+2n+3}$ corresponds to a clique with ${z}$ vertices.

Now, let us consider using a solver for Problem~\ref{problem:nonbinarydisco} to solve the max clique problem. This is discussed by cases.

[Case 1] If the answer to Problem~\ref{problem:nonbinarydisco} is ``no,'' it means that there does not exist a tree (either a clique tree or a general tree) $T$ satisfying $\sum_{i: u_i \in {V}} F_1({T}, {T}_i) \geq \tilde z$. Thus, there does not exist a clique $Q$ satisfying $|Q|\geq z$.

[Case 2] If the answer to Problem~\ref{problem:nonbinarydisco} is ``yes,'' there must exist a tree $T$ satisfying the condition $\sum_{i: u_i \in {V}} F_1({T}, {T}_i) \geq \tilde z$. Notice that such a tree is not guaranteed to be a clique tree, but in this case, we can show there must also exist a clique tree satisfying the condition. 

In general, a tree is a clique tree if and only if it does not include any zero-hit constituents (zero hit means not appearing in any $T_i$). This is because, other than single words and the whole-sentence constituent, a clique tree selects a subset of $\{S_i\}_{i:u_i \in {V}}$, whereas each $T_i$ only selects one element of this set. Let $C_0(T)$ denote zero-hit constituents in $T$. We construct $T'$ by removing $C_0(T)$ from $T$.  We have
\begin{align}
    \tilde z &\leq \sum_{i: u_i \in {V}} F_1(T, T_i) 
    \label{eq:npproblemholds}\\
    &= \sum_{i: u_i \in {V}} \frac{2\sum_{c \in C(T)} \mathds{1}[c \in C(T_i)]}{|C(T)|+|C(T_i)|} \\
    &\leq \sum_{i: u_i \in {V}} \frac{2\sum_{c \in C(T)\backslash C_0(T)} \mathds{1}[c \in C(T_i)]}{|C(T)|-|C_0(T)|+|C(T_i)|}
    \label{eq:removingC0}\\
    &= \sum_{i: u_i \in {V}} \frac{2\sum_{c \in C(T')} \mathds{1}[c \in C(T_i)]}{|C(T')|+|C(T_i)|} \\
    &= \sum_{i: u_i \in {V}} F_1(T', T_i)
\end{align}
Here, Inequality~\eqref{eq:npproblemholds} is due to the assumption of Case 2 that the answer to Problem~\ref{problem:nonbinarydisco} is ``yes.'' Inequality~\eqref{eq:removingC0} is because $|C_0(T)|\ge0$ and $h(c)=0$ for $c\in C_0(T)$, potentially increasing the denominator while the numerator is the same.

Since $T'$ is a clique tree that satisfies $\sum_{i: u_i \in {V}} F_1(T', T_i)\ge \tilde z$, we find the solution to the max clique problem is also ``yes.''

The given reduction can be done in polynomial time, as it iterates over all vertices and vertex pairs. This concludes that Problem~\ref{problem:nonbinarydisco} belongs to \textbf{NP}-complete.
\end{proof}

\vspace{1pt}
\subsection{Proof of Lemma~\ref{lemma:NMWC}}
\label{Apndx:proof:lemmaNMWC}

\noindent\textbf{Lemma 1.} \begin{it}
    Given a graph $G$ as described in \S\ref{sec:our_search_algo}, the clique $Q^*$ maximizing $f(Q; \alpha_1, \alpha_2)$ corresponds to a constituency tree, if $\alpha_1 \leq K\alpha_2$ where $K$ is the number of individuals.
\end{it}
\begin{proof}
A selection of constituents corresponds to a constituency tree over a sentence if and only if (1) it includes all the single words and the whole-sentence constituent, referred to as \textit{trivial} constituents, and (2) all of its constituents are compatible with each other (already satisfied by the clique definition). 

Trivial constituents are compatible with any possible constituent, and thus their corresponding vertices are connected to all vertices in $G$. Moreover, the hit count of such a constituent equals $K$, the number of individual trees, because trivial constituents appear in all constituency trees.

We assume by way of contradiction that a trivial constituent does not appear in the constituency tree corresponding to $Q^*$. This is equivalent to having a vertex $u$---connected to every vertex in $G$ with $w(u)=K$, i.e., having a weight of $K$---not appearing in $Q^*$; that is,
\begin{align}
    f({Q}^* \cup \{u\}; \alpha_1, \alpha_2) &< f({Q}^*; \alpha_1, \alpha_2)
\end{align}
With the definition of $f$, we have
\begin{align}
    \frac{\sum_{v\in{Q}^*} w(v){+}w(u){+}\alpha_1}{|{Q}^*|+1+\alpha_2} &< \frac{\sum_{v\in{Q}^*} w(v) + \alpha_1}{|{Q}^*|+\alpha_2}\\
    \frac{\sum_{v\in{Q}^*} w(v){+}w(u){+}\alpha_1}{\sum_{v\in{Q}^*} w(v) + \alpha_1} &< \frac{|{Q}^*|+1+\alpha_2}{|{Q}^*|+\alpha_2}\\
    1+\frac{w(u)}{\sum_{v\in{Q}^*} w(v) + \alpha_1} &< 1+\frac{1}{|{Q}^*|+\alpha_2}\\
    w(u) &< \frac{\sum_{v\in{Q}^*} w(v) + \alpha_1}{|{Q}^*|+\alpha_2} \\
    w(u) &< \frac{K|{Q}^*| + K\alpha_2}{|{Q}^*|+\alpha_2} = K
\end{align}
Here, the last inequality is due to $\alpha_1 < K\alpha_2$ and $w(v) \leq K$ for every $v$. But this inequality contradicts our assumption that $w(u)=K$.
\end{proof}

\subsection{Proof of Theorem~\ref{thm:lowerbound}}
\label{Apndx:proof:lowerbound}

\theoremlowerbound*

\begin{proof}
Consider any constituent $c_0\in C(T^*)\backslash P$. We construct $T'$ by removing $c_0$ from $T^*$, i.e., 
$C(T^*) = C(T') \cup \{c_0\}$. Based on Eqn.~\eqref{eq:SimpliedAvg}, we have
\begin{align}
\frac{\sum_{c\in C(T^*)} {\mathpzc h}(c)}{|C(T^*)|+2n-1} &> 
\frac{\sum_{c\in C(T')} {\mathpzc h}(c)}{|C(T')|+2n-1}\\
\frac{\sum_{c\in C(T^*)} {\mathpzc h}(c)}{\sum_{c\in C(T')} {\mathpzc h}(c)} &> 
\frac{|C(T^*)|+2n-1}{|C(T')|+2n-1}\\
1+\frac{{\mathpzc h}(c_0)}{\sum_{c\in C(T')} {\mathpzc h}(c)} &>
1+\frac{1}{|C(T')|+2n-1}\\
{\mathpzc h}(c_0) &>
\frac{\sum_{c\in C(T')} {\mathpzc h}(c)}{|C(T')|+2n-1}
\label{eq:tprimedependentlowerbound}
\end{align}

Since the right-hand side is non-negative, we must have $h(c_0)> 0$. In other words, zero-hit constituents do not appear in $C(T^*)\backslash P$. Thus, $m$-many constituents in $C(T^*)\backslash P$ have a total of hit count greater than or equal to $\sum_{i=1}^{m} \lambda_i^{(+)}$, because $\lambda_i^{(+)}$ is the $i$th smallest positive hit count. 

On the other hand, we have $c_0\notin P$, implying that $P \subseteq C(T')$. Since a constituent $c\in C(T')$ is either in $P$ or not, we can lower-bound the total hit count of $T'$ by 
\begin{align}
    \sum_{c\in C(T')} {\mathpzc h}(c) \geq \sum_{c'\in P} {\mathpzc h}(c')+ \sum_{i=1}^{|C(T')\backslash P|} \lambda^{+}_i
    \label{eq:lowerusinglambda}
\end{align}
Putting~\eqref{eq:lowerusinglambda} into~\eqref{eq:tprimedependentlowerbound}, we have
\begin{align}
{\mathpzc h}(c_0) &>
\frac{
\sum_{i=1}^{|C(T')|-|P|} \lambda^{+}_i
+
\sum_{c'\in P} {\mathpzc h}(c')
}{|C(T')|+2n-1}
\label{eq:tprimedependentlowerboundusinglambda}
\end{align}
Moreover, $P \subseteq C(T')$ also implies
\begin{align}
    |P| \leq |C(T')| = |C(T^*)|-1 \leq 2n-2
\end{align}
which derives~\eqref{eq:tprimedependentlowerboundusinglambda} to a $T'$-independent lower bound for $h(c_0)$ if we consider all the values $|C(T')|$ can take, given by
\begin{align}
h(c_0) >
    \min_{|P| \leq j \leq 2n-2} \frac{\sum_{i=1}^{j-|P|} \lambda^{+}_i + \sum_{c'\in P} {\mathpzc h}(c')}{j+2n-1}
\end{align}
concluding the proof.
\end{proof}
The lower bound can be applied for pruning the search. We may immediately prune zero-hit constituents by setting $P=\emptyset$, and more importantly, it yields a much tighter lower bound together with Theorem~\ref{thm:upperbound}.

\subsection{Proof of Theorem~\ref{thm:upperbound}}
\label{Apndx:proof:upperbound}

\theoremupperbound*

\begin{proof}[]
[Part (a)] A hit count of $K$ indicates that the constituent appears in all the individuals. Therefore, $c$ is compatible with every constituent in every individual.
On the other hand, Theorem~\ref{thm:lowerbound} shows that every constituent in the average tree has a positive hit count (appears in at least one individual) and thus is compatible with $c$.

[Part (b)]
We assume by way of contradiction that a constituent $c_0$ with a hit count of $K$ does not appear in the average tree $T^*$. As shown in (a), $c_0$ is compatible with every constituent in $T^*$; hence, there exists a constituency tree $\hat{T}$ such that $C(\hat{T}) = C(T^*) \cup \{c_0\}$. Based on Eqn.~\eqref{eq:SimpliedAvg}, we have
\begin{align}
\frac{\sum_{c\in C(\hat{T})} {\mathpzc h}(c)}{|C(\hat{T})|+2n-1} &< 
\frac{\sum_{c\in C({T^*})} {\mathpzc h}(c)}{|C(T^*)|+2n-1}\\
\frac{\sum_{c\in C(\hat{T})} {\mathpzc h}(c)}{\sum_{c\in C({T^*})} {\mathpzc h}(c)} &<
\frac{|C(\hat{T})|+2n-1}{|C(T^*)|+2n-1}\\
1+\frac{{\mathpzc h}(c_0)}{\sum_{c\in C({T^*})} {\mathpzc h}(c)} &<
1+\frac{1}{|C(T^*)|+2n-1}\\
{\mathpzc h}(c_0) &<
\frac{\sum_{c\in C({T^*})} {\mathpzc h}(c)}{|C(T^*)|+2n-1}\\
{\mathpzc h}(c_0) &<
\frac{K|C(T^*)|}{|C(T^*)|+2n-1} \leq K
\end{align}
Here, the last inequality is due to $\mathpzc h(c) \leq K$ for every $c$. But this contradicts our assumption that $\mathpzc h(c_0)=K$.
\end{proof}

\section{Our Meet-in-the-Middle Algorithm}
\label{apndx:meetinthemiddle}

The meet-in-the-middle technique~\citep{horowitz1974computing} has been used to tackle the max clique problem. This method halves the exponent of the time complexity of exhaustive search, cutting it down from $\mathcal{O}(2^{|V|})$ to $\mathcal{O}(2^\frac{|V|}{2}|V|^2)$ given a graph with $|V|$ vertices. However, the exact same algorithm is not sufficient for our needs, because Problem~\ref{problem:NMWC} is a generalization of the max clique problem and requires additional consideration. In this section, we develop an extended version of the algorithm that solves Problem~\ref{problem:NMWC} within the same time complexity.

\problemNMWC*

To find the solution to Problem~\ref{problem:NMWC}, we utilize the meet-in-the-middle technique, which splits the vertex set $V$ into two parts $V_1$ and $V_2$, whose sizes are as equal as possible. 

For every $V'_1 \subseteq V_1$ and every $0 \le j \le |V'_1|$, we define 
\begin{align}
\operatorname{best}(V'_1, j) = 
\operatorname*{argmax}\limits_{\substack{V_1'':\ V_1'' \subseteq V'_1, |V''_1|=j,\\\text{$V''_1$ is a clique}}}\quad \sum_{v\in V''_1} w(v)
\end{align}
which is the best $j$-vertex clique in $V_1'$, and can be computed recursively.

For initialization, we have $\operatorname{best}(V'_1, 0)=\emptyset$ for every $V'_1\subseteq V_1$ and $\operatorname{best}(\{u\}, 1) = \{u\}$ for every $u \in V_1$.

The recursion is to compute the $\operatorname{best}$ set for any multi-vertex $V'_1$ assuming we have the $\operatorname{best}$ set for every subset of $V_1$, whose size is smaller than $V'_1$. To achieve this, we pick any $u \in V'_1$, and the $\operatorname{best}(V'_1, j)$ may or may not include $u$, which is discussed by cases. 

If $u \not\in \operatorname{best}(V'_1, j)$,  we have
\begin{align}\label{eq:case1}
\operatorname{best}(V'_1, j)=\operatorname{best}(V'_1\backslash\{u\}, j)
\end{align}

If otherwise $u \in \operatorname{best}(V'_1, j)$, then the other $j-1$ vertices must be the neighbors of $u$ in $V'_1$, because $\operatorname{best}(\cdot, \cdot)$ must be a clique. 
Denoting by $A_{V'_1}^{(u)}$ all the vertices in $V'_1$ connected to $u$, we have
\begin{align}\label{eq:case2}\operatorname{best}(V'_1, j) = \{u\}\cup\operatorname{best}(A_{V'_1}^{(u)}, j-1)
\end{align}

Combining the two cases, we have $\operatorname{best}(V'_1, j)$ be either Eqn.~\eqref{eq:case1} or Eqn.~\eqref{eq:case2}, depending on which yields a higher weight. 

For the $V_2$ part, we would like to evaluate every clique $V'_2\subseteq V_2$ and its corresponding best match in $V_1$ to obtain $Q^*\subseteq V=V_1\cup V_2$ that maximizes our scoring function $f(Q;\alpha_1, \alpha_2)$.

We enumerate all the subsets of $V_2$. For every clique $V'_2 \subseteq V_2$, we denote by $A_{V_1}^{(V'_2)}$ the vertices in $V_1$ that are connected to every vertex in $V'_2$. We compute
\begin{align}
\operatorname{val}(V'_2, j) = \sum_{v\ \in\ V'_2\; \cup\; \operatorname{best}\!\big(A_{V_1}^{(V'_2)}, j\big)} w(v)\quad
\end{align}
being the weight of the best clique, which in $V_1$ has $j$ vertices and in $V_2$ involves $V'_2$ only.

Since the $V=V_1\cup V_2$, we can find the vertices of $Q^*$ in $V_1$ and $V_2$ separately:
\begin{align}
Q^*_2,\; j^* =& \operatorname*{argmax}_{V'_2,\; j} \frac{\operatorname{val}(V'_2, j) + \alpha_1}{|V'_2|+j+\alpha_2}\\
Q^*_1 =& \operatorname{best}\!\big(A_{V_1}^{(Q^*_2)}, j^*\big)
\end{align}
The final answer is
$Q^* = Q^*_1 \cup Q^*_2$.

We analyze the time complexity of our algorithm. For the $V_1$ part, our algorithm enumerates $\mathcal O(2^{\frac{|V|}2})$-many subsets $V_1'$ and $\mathcal O(|V|)$-many $j$ values; in each iteration, $\mathcal O(|V|)$-many vertices are visited for $A^{(u)}_{V'_1}$. For the $V_2$ part, we enumerate $\mathcal O(2^\frac{|V|}{2})$-many subsets of $V_2$, each verified for being a clique in $\mathcal O(|V|^2)$ time; after that, $A_{V_1}^{(V_2')}$ is retrieved with complexity $\mathcal O(|V_1|\cdot |V_2'|)=\mathcal O(|V|^2)$, which can be absorbed in the $\mathcal O$-notation. Finally, retrieving $Q^*$ happens in $\mathcal O(2^\frac{|V|}{2}|V|)$ time, and the overall time complexity of our algorithm is $\mathcal O(2^\frac{|V|}{2}|V|^2)$.

\begin{figure}[t]
\begin{center}
\includegraphics[width=\linewidth]{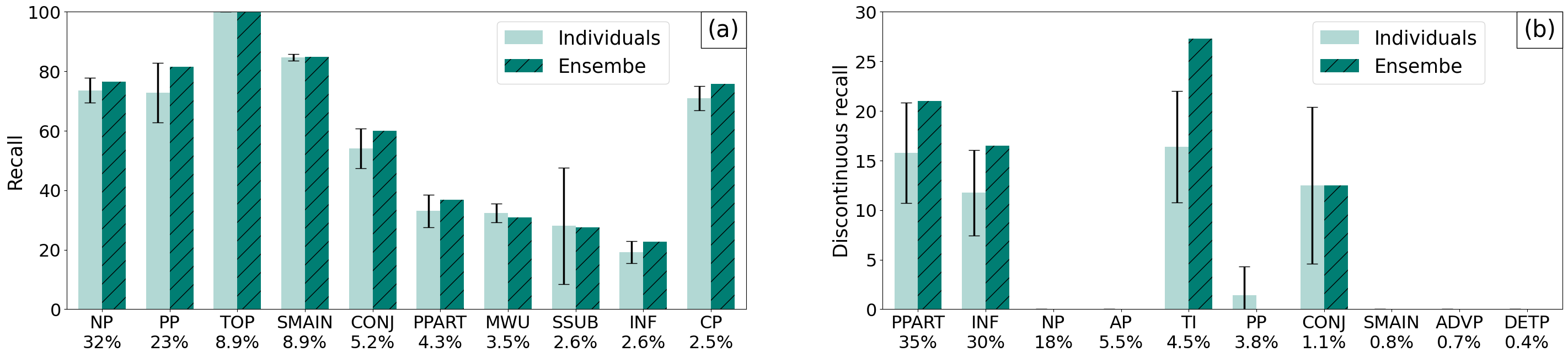}
\end{center}\vspace{-10pt}
\caption{Recall scores of 10 most frequent constituent types on LASSY: (a) among all constituents, and (b) among discontinuous constituents. The percentages under the types indicate their occurrences in the dataset. Error bars are the standard deviation for five individuals.}
\label{fig:per_label_analysis}
\end{figure}

\section{Performance by Constituency Types}
\label{apndx:per_label_analysis}

We would like to analyze our ensemble method's performance from a linguistic point of view. We provide a performance breakdown by constituency types, e.g., noun phrases (NP) and verb phrases (VP). Note that our unsupervised models offer unlabeled constituency structures, and thus we can only compute recall scores for different constituency types. 

In Figures~\ref{fig:per_label_analysis}a and \ref{fig:per_label_analysis}b, we report the recall scores for all constituents and discontinuous constituents, respectively. We can see for most of the types, the ensemble can retain the best individual's performance in that type, suggesting that the ensemble is able to leverage the diverse strengths of different individuals.

\section{Case Studies}
\label{appndx:case}

In Figure~\ref{fig:case1}, we present a case study on the LASSY dataset to show how the ensemble can take advantage of different individuals. In the figure, a groundtruth constituent is annotated by its type (e.g., \texttt{CONJ} representing conjunction). However, our unsupervised parsing is untyped, and we use~``\texttt{o}'' to denote a constituent. A skipped line  ``\texttt{-|-}'' indicates discontinuity.

Consider the last three words ``\textit{moeten worden opgesteld}'' (must be established).\footnote{We provide English interpretations by ChatGPT-4 with the prompt: ``\texttt{Translate the following Dutch sentence. Give word-by-word literal translation.''}} We see the ensemble output follows the majority votes  (three out of five, namely Individuals~$2$,~$4$, and~$5$) and detects the correct structure of this phrase. Moreover, the ensemble is able to detect the discontinuous constituent ``\textit{in sommige gevallen ... niet toereikend zijn}'' ({in some cases ... not sufficient be}), suggested by Individuals~$4$ and~$5$. On the other hand, we also notice that every individual produces incorrect discontinuous constituents, but our ensemble is able to smooth out such noise and achieve high $F_1$ scores.

\begin{figure*}[t]
\hspace{20pt}
\includegraphics[width=1.7\linewidth, height=\textheight]{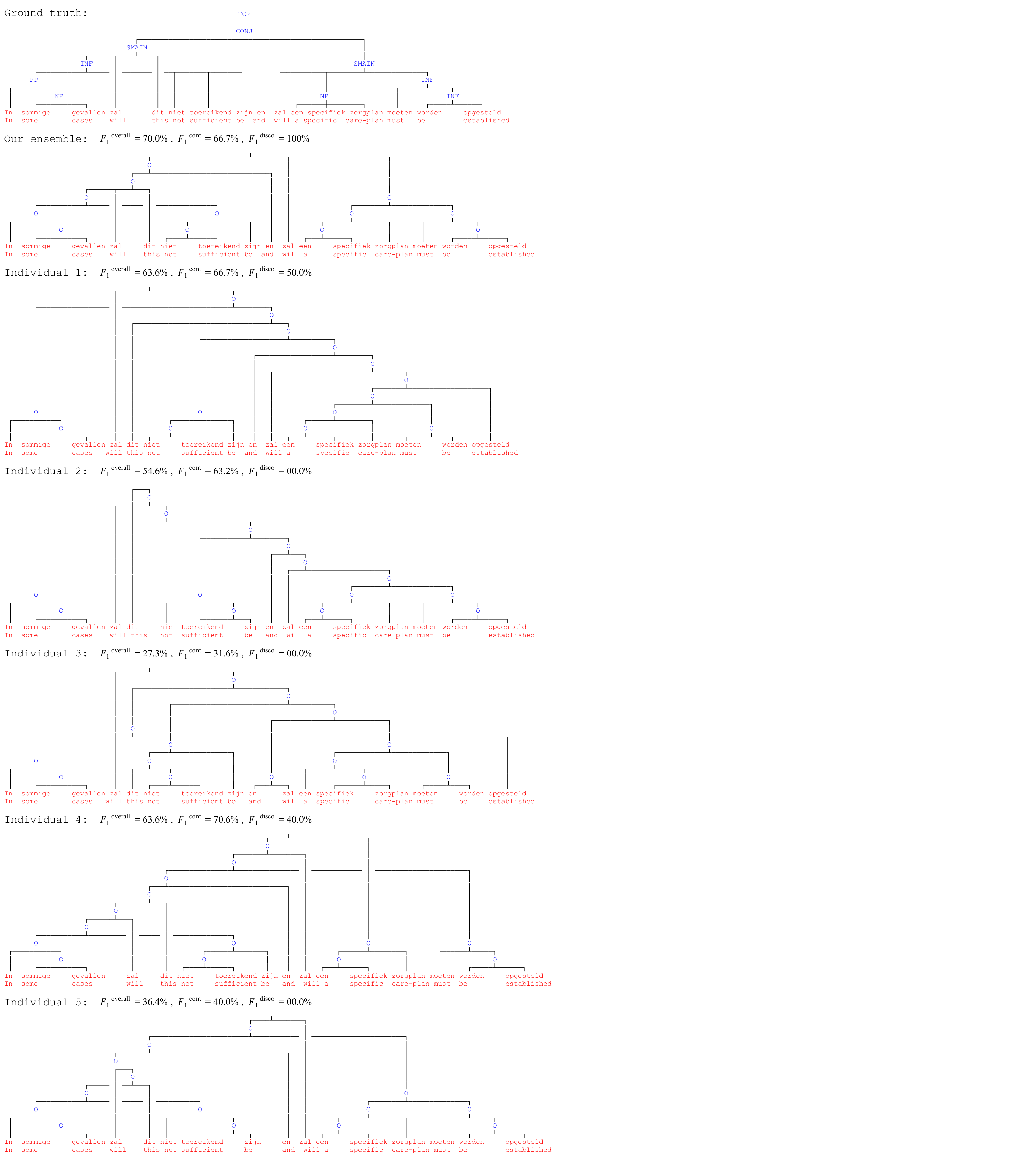}
\vspace{-25pt}
\caption{Case studies with an example in Dutch from the LASSY dataset.}
\label{fig:case1}
\end{figure*}

\end{document}